\documentclass[smallextended]{svjour3}       
\smartqed  
\usepackage{graphicx}
\usepackage{amsmath,amssymb}
\usepackage{color}
\usepackage{hyperref}
\journalname{Submitted to Ann Oper Res}
\begin{document}

\title{Multi-Sensor Slope Change Detection
}
\author{Yang Cao        \and
        Yao Xie \and Nagi Gebraeel
}


\institute{Yang Cao, Yao Xie, and Nagi Gebraeel \at
              H. Milton Stewart School of Industrial and Systems Engineering \\
              Georgia Institute of Technology, Atlanta, Georgia, USA.\\
              \email{\{caoyang, yao.xie, nagi\}@gatech.edu}           
}

\date{Received: date / Accepted: date}

\maketitle

\begin{abstract}
We develop a mixture procedure for multi-sensor systems to monitor data streams for a change-point that causes a gradual degradation to a subset of the streams. Observations are assumed to be initially normal random variables with known constant means and variances. After the change-point,  observations in the subset will have increasing or decreasing means. The subset and the rate-of-changes are unknown. Our procedure uses a mixture statistics, which assumes that each sensor is affected by the change-point with probability $p_0$. Analytic expressions are obtained for the average run length (ARL) and the expected detection delay (EDD) of the mixture procedure, which are demonstrated to be quite accurate numerically. We  establish the asymptotic optimality of the mixture procedure.  Numerical examples demonstrate the good performance of the proposed procedure. We also discuss an adaptive mixture procedure using empirical Bayes. This paper extends our earlier work on detecting an abrupt change-point that causes a mean-shift, by tackling the challenges posed by the non-stationarity of the slope-change problem.
\end{abstract}

\section{Introduction}
\label{intro}

As an enabling component for modern intelligent systems, multi-sensory  monitoring has been widely deployed for large scale systems, such as manufacturing systems \cite{fang2015adaptive}, \cite{6496166}, power systems \cite{VVV-IWSM-2015}, and biological and chemical threat detection systems \cite{threat_CSL_2015}. The sensors acquire a stream of observations, whose distribution  changes when the state of the network is shifted due to an abnormality or threat. We would like to detect the change online as soon as possible after it occurs, while controlling the false alarm rate. When the change happens, typically only a small subset of sensors are affected by the change, which is a form of sparsity. A mixture statistic which utilizes this sparsity structure of this problem is presented in \cite{xie2013sequential}. The asymptotic optimality of a related mixture statistic is established in \cite{optmal_mix_procedure}. Extensions and modifications of the mixture statistic that lead to optimal detection are considered in \cite{ChanOptimal2015}.

In the above references \cite{xie2013sequential,ChanOptimal2015}, the change-point is assumed to cause a shift in the means of the observations by the affected sensors, which is good for modeling an abrupt change. However, in many applications above, the change-point is an onset of system degradation, which causes a gradual change to the sensor observations. Often such a gradual change can be well approximated by a \emph{slope} change in the means of the observations. One such example is shown in Fig. \ref{Fig:NASA_data}, where multiple sensors monitor an aircraft engine  and each panel of figure shows the readings of one sensor.  At some time a degradation initiates and causes decreasing or increasing in the means of the observations. Another example comes from power networks, where there are thousands of sensors monitoring hundreds of transformers in the network. We would like to detect the onset of any degradation in real-time and predict the residual life time of a transformer before it breaks down and causes a major power failure.

In this paper, we present a mixture procedure that detects a change-point causing a slope change to the means of the observations, which can be a model for gradual degradations. Assume the observations at each sensor are $i.i.d.$ normal random variables with constant means. After the change, observations at the sensors affected by the change-point become normal distributed with increasing or decreasing means. The subset of sensors that are affected are unknown. Moreover, their rate-of-changes are also unknown. Our mixture procedure assumes that each sensor is affected with probability $p_0$ independently, which is a guess for the true fraction $p$ of sensors affected. When $p_0$ is small, this captures an empirical fact that typically only a small fraction of sensors are affected. With such a model, we derive the log-likelihood ratio statistic, which becomes applying a soft-thresholding to the local statistic at each sensor and then combining the results. The mixture procedure fires an alarm whenever the statistic exceeds a prescribed threshold.
We consider two versions of the mixture procedure that compute the local sensor statistic differently: the \emph{mixture CUSUM procedure} $T_1$, which assumes some nominal values for the unknown rate-of-change parameters, and the \emph{mixture generalized likelihood ration (GLR) procedure} $T_2$, which uses the maximum likelihood estimates for these parameters.
To characterize the performance of the mixture procedure, we present theoretical approximations for two commonly used performance metrics, the average run length (ARL) and the expected detection delay (EDD). Our approximations are shown to be highly accurate numerically and this is useful in choosing a threshold of the procedure. We also establish the asymptotic optimality of the mixture procedures.
Good performance of the mixture procedure is demonstrated via real-data examples, including: (1) detecting a change in the trends of  financial time series; (2) predicting the life of air-craft engines using the Turbofan engine degradation simulation dataset.

The mixture procedure here can be viewed as an extension of our earlier work on multi-sensor mixture procedure for detecting mean shifts \cite{xie2013sequential}. The extensions of theoretical approximations to EDD and especially to ARL are highly non-trivial, because of the non-i.i.d. distributions in the slope change problem. Moreover, we also establish some new optimality results which were omitted from \cite{xie2013sequential}, by extending the results in  \cite{lai1998information} and \cite{tartakovsky2014sequential} to handle non-$i.i.d.$ distributions in our setting. In particular, we generalize the theory to a scenario where the log likelihood ratio grows polynomially as a result of linear increase or decrease of the mean values, whereas in  \cite{xie2013sequential}, the log-likelihood ratio grows linearly. A related recent work \cite{optmal_mix_procedure} studies optimality of the multi-sensor mixture procedure for $i.i.d.$ observations, but the results therein do not apply to the slope change case here.

The rest of this paper is organized as follows. Section \ref{sec:formulation} sets up the formalism of the problem. Section \ref{sec:procedures} presents our mixture procedures for slope change detection, and Section \ref{sec:theoretical} presents theoretical approximations to its ARL and EDD, which are validated by numerical examples. Section \ref{sec:optimality} establishes the first order asymptotic optimality. Section \ref{sec:num_eg} shows real-data examples. Finally, Section \ref{sec:con} presents an extension of the mixture procedure that adaptively chooses $p_0$. All proofs are delegated to the appendix.

\section{Assumptions and formulation}\label{sec:formulation}

Given $N$ sensors. For the $n$th sensor $n = 1, 2, \ldots, N$, denote the sequence of observations   by $y_{n,i}$, $i=1,2,\ldots$. 
Under the hypothesis of no change, the  observations at the $n$th sensor have a \emph{known} mean $\mu_n$ and a \emph{known} variance $\sigma_n^2$. Probability and expectation in this case are denoted by $\mathbb{P}_\infty$ and $\mathbb{E}_\infty$, respectively. Alternatively, there exists an \emph{unknown} change-point that occurs at time $\kappa$, $0\leq \kappa < \infty$, and it affects an \emph{unknown} subset $\mathcal{A}\subseteq\{1, 2, \ldots, N\}$ of sensors  \emph{simultaneously}.  The fraction of affected sensors is given by $p = |\mathcal{A}|/N$.  For each affected sensor $n\in \mathcal{A}$, the mean of the observations $y_{n, t}$ changes linearly from the change-point time $\kappa+1$ and is given by $\mu_n + c_n (t-\kappa)$ for all $t > \kappa$, and the variance remains $\sigma_n^2$. For each unaffected sensor, the distribution stays the same. Here the \emph{unknown} rate-of-change $c_n$ can differ across sensors and it can be either positive or negative. The probability and expectation in this case are denoted by $\mathbb{P}_{\kappa}^{\mathcal{A}}$ and $\mathbb{E}_{\kappa}^{\mathcal{A}}$, respectively. In particular, $\kappa = 0$ denotes an immediate change occurring at the initial time.
The above setting can formulate as the following hypothesis testing problem:
\begin{equation}
\begin{split}
H_{0}: \quad& y_{n,i} \sim \mathcal{N}(\mu_n, \sigma_n^2), i=1,2,\ldots,n=1,2,\ldots,N, \\
H_{1}: \quad& y_{n,i} \sim \mathcal{N}(\mu_n, \sigma_n^2), i=1,2,\ldots,\kappa, \\
      \quad& \quad y_{n,i} \sim \mathcal{N}(\mu_n + c_n(i-\kappa), \sigma_n^2), i=\kappa+1,\kappa+2,\ldots,
       n\in \mathcal{A}, \\
      \quad&y_{n,i} \sim \mathcal{N}(\mu_n, \sigma_n^2), i=1,2,\ldots, n \in \mathcal{A}^c.
\label{P1}
\end{split}
\end{equation}
Our goal is to establish a stopping rule that stops as soon as possible after a change-point occurs and avoids raising false alarms when there is no change. We will make these statements more rigorous in Section \ref{sec:theoretical} and Section \ref{sec:optimality}.
Here, for simplicity, we assume that all sensors are affected by the change simultaneously. This ignores the fact that there can be delays across sensors. For asynchronous sensors, one possible approach is to adopt the scheme in \cite{hadjiliadis2009one}, which claims a change-point whenever the any sensor detects a change. We plan investigate the issue of delays in our future work.

A related problem is to detect a change in a linear regression model. One such example is a change-point in the trend of the stock price illustrated in Fig. \ref{fig:finance}(a). This can be casted into a slope change detection problem, if we fit a linear regression model under $H_0$ (e.g., using historical data) and subtract it from the sequence. The residuals after the subtraction will have zero means before the change-point, and their means  will increase or decrease linearly after the change-point.

\begin{figure}
\begin{center}
\includegraphics[width = 1.0\textwidth]{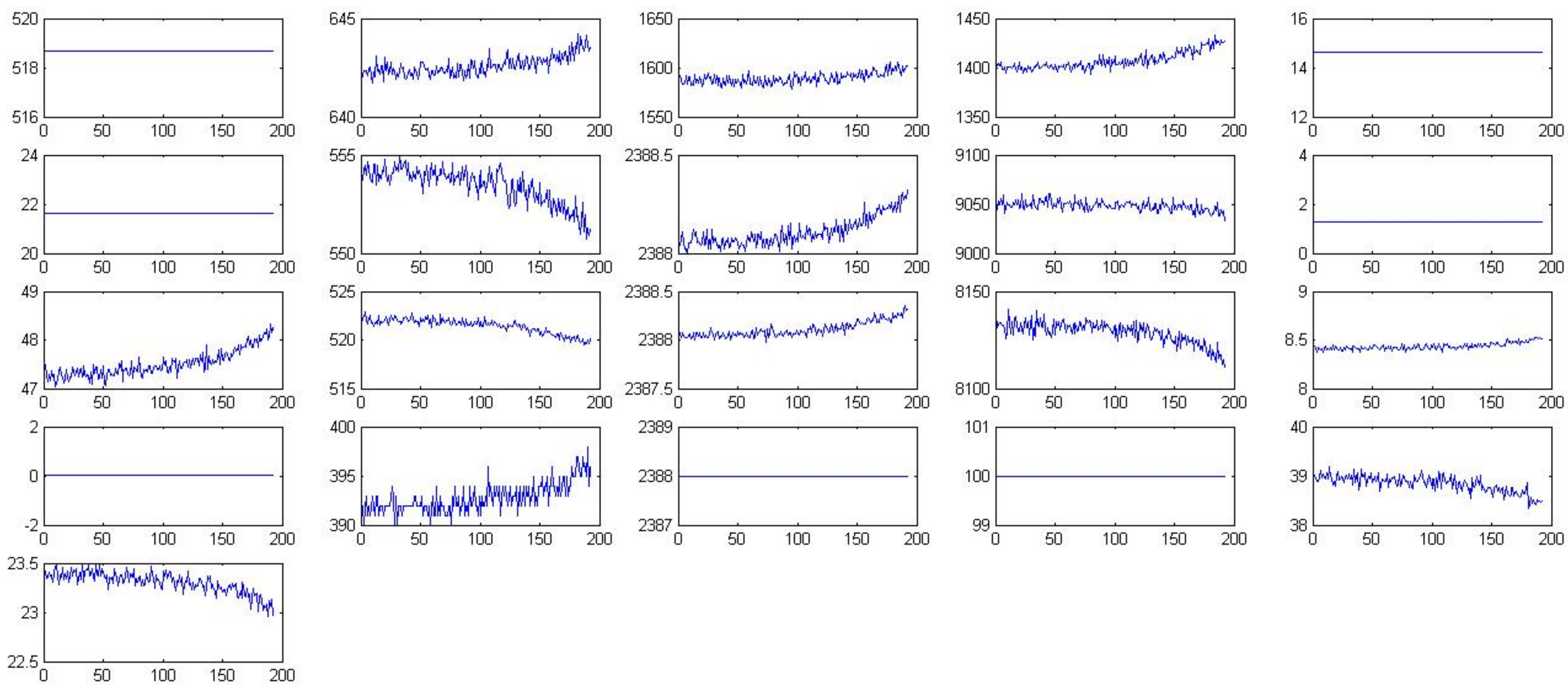}
\end{center}
\caption{Degradation sample paths recorded by 21 sensors, generated by C-MAPSS \cite{NASA_engine_simu}. A subset of sensors are affected by the change-point, which happens at an unknown time simultaneously and it causes a change in the slopes of the signals. The change can cause either an increase or decrease in the means.  
}
\label{Fig:NASA_data}
\end{figure}

\begin{remark}[Reason for not differencing the signal.] One legitimate question is that why not de-trending the signal at each sensor by difference, which may turn the slope change into a mean change problem and we can apply the standard CUSUM procedure designed for detecting the mean shift. Indeed, for the affected sensors after the change-point, $\mathbb{E}[y_{n, i+1} - y_{n, i}] = c_n$. However, differencing will also increase the variance, as $\mbox{Var}[y_{n, i+1} - y_{n, i}] = 2\sigma_n^2$. Hence, differencing reduces the signal-to-noise ratio and this is particularly bad for weak signals and makes them even non-detectable. This is validated by real data as well. Consider the engine data displayed in Fig. \ref{Fig:NASA_data}.  The first panel in Fig. \ref{fig:comparison_diff} corresponds to observations of one sensor that is affected by noise, which clearly has the ``signal'' as the mean is increasing. However, after we difference the signal, the change is almost invisible, as illustrated. We then try applying CUSUM on the difference, where the statistic either rises slowly which means it cannot detect quickly. We also try applying CUSUM (designed for mean change detection) directly on the original signal. Although the statistics rises reasonably fast; however, clearly this statistic cannot estimate the change-point location accurately due to its model mismatch. The last panel in Fig. \ref{fig:comparison_diff} shows our proposed statistic, which can detect the change fairly quickly, and it can also accurately estimate the change-point location.
%
\end{remark}

\begin{figure}[h]
\begin{center}
\includegraphics[width = 0.7\linewidth]{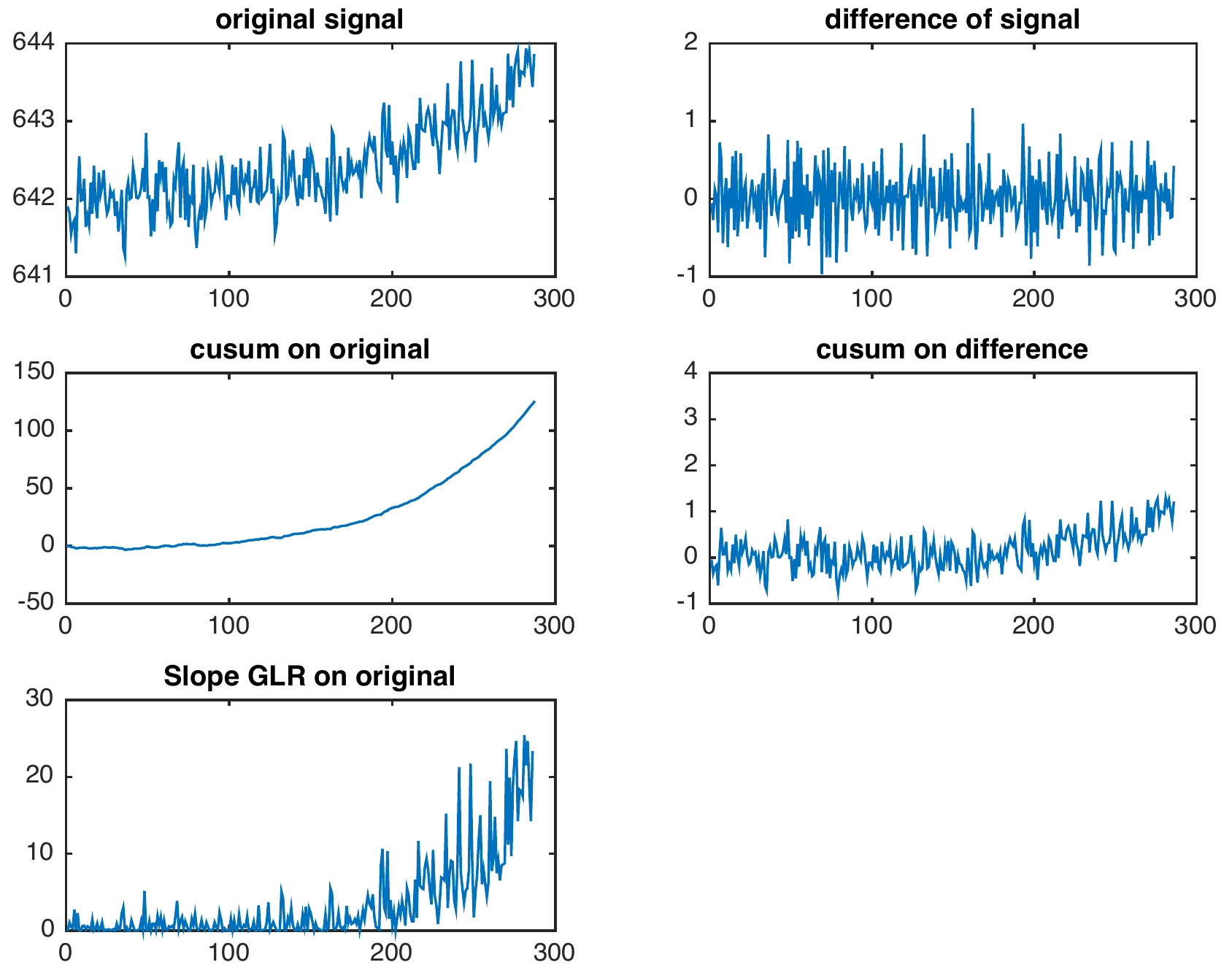}
\caption{Engine data displayed in Fig. \ref{Fig:NASA_data}: the first panel corresponds to observations of one sensor that is affected by noise, which clearly has the ``signal'' as the mean is increasing. However, after we difference the signal, the change is almost invisible. Then we compare applying CUSUM procedure (designed for mean shift) on the original signal, on the difference, and applying our proposed statistic on the original signal. }
\label{fig:comparison_diff}
\end{center}
\vspace{-.1in}
\end{figure}



\section{Detection procedures}\label{sec:procedures}
Since the observations are independent, for an assumed change-point location $\kappa = k$ and an affected sensor $n \in \mathcal{A}$, the log-likelihood for observations up to time $t > k$ is given by
\begin{equation}
\ell_n(k, t, c_n)
=\frac{1}{2\sigma_n^2}\sum_{i=k+1}^t \left[ 2c_n(y_{n,i}-\mu_n)(i-k) - c_n^2(i-k)^2 \right].
\label{likelihood}
\end{equation}
Motived by the mixture procedure in \cite{xie2013sequential} and \cite{siegmund2011detecting} to exploit an empirical fact that typically only a subset of sensors are affected by the change-point, we assume that each sensor is affected with probability $p_0 \in (0, 1]$ independently. In this setting, the log likelihood of all $N$ sensors is given by
\begin{equation}
\sum_{n=1}^N \log\left( 1-p_0+p_0\exp\left[\ell_n(k,t,c_n) \right] \right).
\label{global}
\end{equation}
Using (\ref{global}), we may derive several change-point detection rules.

Since the rate-of-change $c_n$ is unknown, One possibility is to set $c_n$ equal to some nominal post-change value $\delta_n$ and define the stopping rule, referred to as the \emph{mixture CUSUM} procedure:
\begin{equation}
T_1 = \inf\left\{t: \max_{0\leq k <t} \sum_{n=1}^N \log \left(1-p_0+p_0 \exp[\ell_n(k,t,\delta_n)] \right) \geq b \right\},
\label{Slope_mixtureCUSUM}
\end{equation}
where $b$ is a threshold typically prescribed to satisfy the average run length (ARL) requirement (formal definition of ARL is given in Section \ref{sec:theoretical}).

Another possibility is to replace $c_n$ by its maximum likelihood estimator. Given  the current number of observations $t$ and a putative change-point location $k$, by setting the derivative of the log likelihood function (\ref{likelihood}) to 0,
we may solve for the maximum likelihood estimator: 
\begin{equation}
\hat{c}_n(k, t) = \frac{\sum_{i=k+1}^t (i-k) (y_{n,i}-\mu_n)}{\sum_{i=k+1}^t (i-k)^2}.
\label{estimation_of_slope}
\end{equation}
Define $\tau = t-k$ to be the number of samples after the change-point $k$.  Denote the sum of squares from $1$ to $\tau$, and the weighted sum of data as, respectively,
\[A_{\tau} = \sum_{i=1}^{\tau} i^2, \qquad
W_{n, k, t} = \sum_{i = k+1}^t (i-k) (y_{n, i}-\mu_n)/\sigma_n.\]
Let
\begin{equation}
\begin{split}
U_{n,k,t} =& \left(A_\tau\right)^{-1/2}  W_{n, k, t}.
\label{U_def}
\end{split}
\end{equation}
Substitution of  (\ref{estimation_of_slope}) into (\ref{likelihood}) gives the log generalized likelihood ratio (GLR) statistic at each sensor:
\begin{equation}
\ell_n(k,t,\hat{c}_n)
= U_{n,k,t}^2/2, \label{local_GLR}
\end{equation}
and we define the \emph{mixture GLR} procedure as
\begin{equation}
T_2 = \inf\left\{t: \max_{0\leq k <t} \sum_{n=1}^N \log \left(1-p_0+p_0 \exp\left[U_{n,k,t}^2/2\right] \right) \geq b \right\},\label{mainstatistic}
\end{equation}
where $b$ is a prescribed threshold.

\begin{remark}[Window limited procedures.]
In the following we use \emph{window limited} versions of $T_1$ and $T_2$, where the maximum for the statistic is restricted to a window $t-w\leq k \leq t-w'$ for suitable choices of window size $w$ and $w'$. In the following, we use $\widetilde{T}$ to denote a window-limited version of a procedure $T$.
%
%
By searching only over a window of the past $w-w'+1$ samples, this reduces the memory requirements to implement the stopping rule, and it also sets a minimum level of change that we want to detect.
%
The choice of $w$ may depend on $b$ and sometimes we need make additional assumptions on $w$ for the purpose of establishing the asymptotic results below. More discussions about the choice of $w$ can be found in \cite{lai1995sequential} and \cite{lai1998information}.
 The other parameter $w'$ is the minimum number of observations needed for computing the maximum likelihood estimator for parameters. In the following, we  set $w'=1$. 

\end{remark}

\begin{figure}
\begin{center}
\includegraphics[width = .5\textwidth]{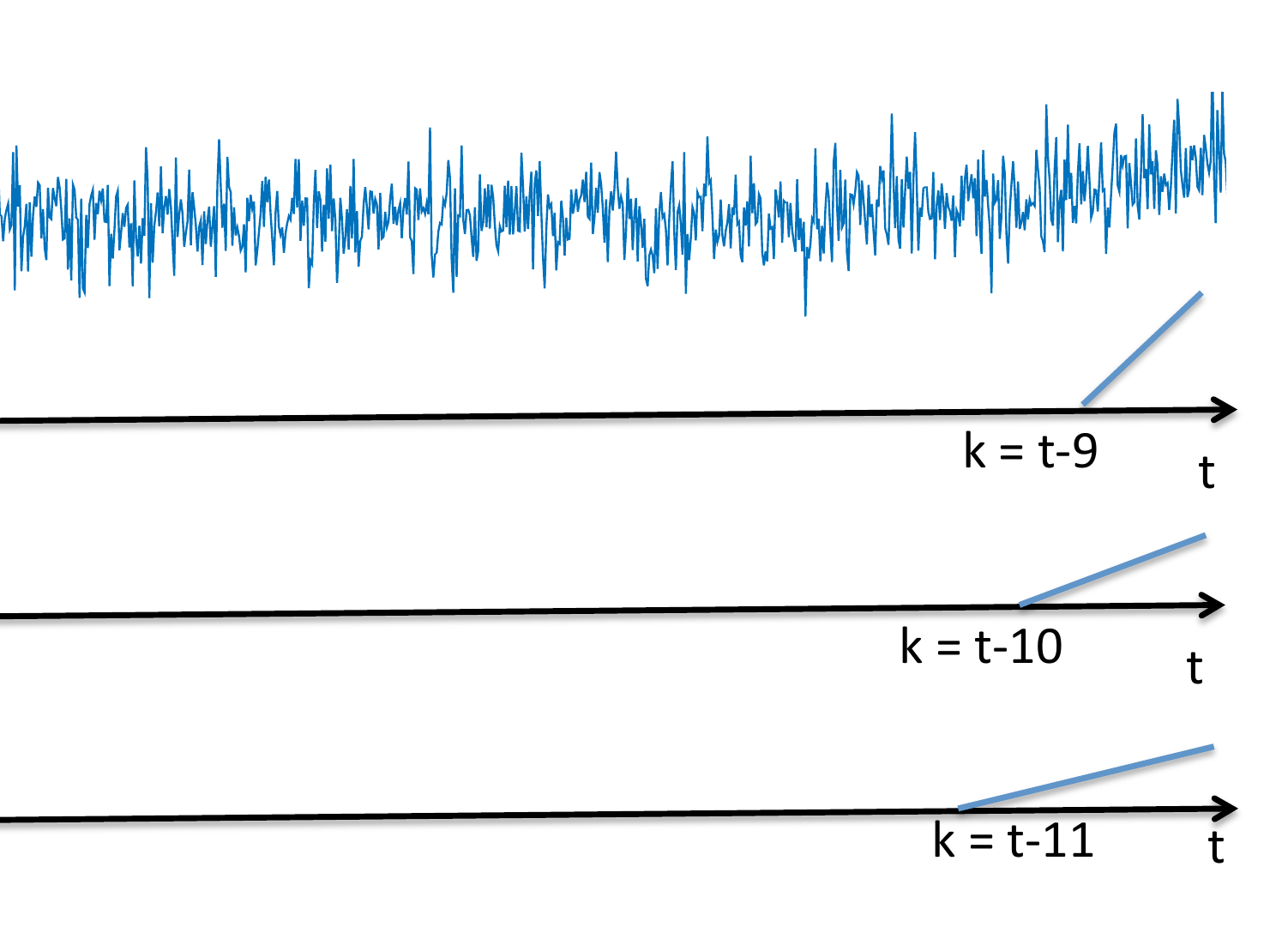}
\end{center}
\caption{Matched filter interpretation of the generalized likelihood ratio statistic at each sensor $U_{n, k, t} = A_{\tau}^{-1} \sum_{i = k+1}^t (i-k) (y_{n, i} - \mu_n)/\sigma_n$: data at each sensor is matched with a triangle-shaped signal that starts at a hypothesized change-point time $k$ and ends at the current time $t$. The slope of the triangle is $A_{\tau}^{-1}$, so that the $\ell_2$ norm of the triangle signal is one. }
\vspace{-0.1in}
\label{fig:MF}
\end{figure}

\begin{remark}[Relation to mean shift.]\label{remark_mean_shift}
For the mean-shift multi-sensor change-point detection \cite{xie2013sequential}, the detection statistic depends on a key quantify, which is the average of the samples in the time window $[k+1, t]$. Note that in the slope change case, the detection statistic has a similar structure, except that the key quantity is replaced by a weighted average of the samples in the window: $(t-k)^{-1/2}\sum_{i=k+1}^t (y_{n, i}-\mu_n)/\sigma_n$. This has an interpretation of ``matched filtering'', as illustrated in Fig. \ref{fig:MF}: each data stream is matched with a triangle shaped signal starting at a potential change-point time $k$ that represents a possible slope change.
\end{remark}

\begin{remark}[Recursive computation.]
The quantity $W_{n,k,t}$ involved in the detection statistic for (\ref{mainstatistic}) can be calculated recursively,
\[W_{n,k,t+1} = W_{n,k,t} + (t+1-k)\left((y_{n,t+1}-\mu_n)/\sigma_n\right),\] where $W_{n,t,t}\triangleq0$. This facilitates online implementation of the detection procedure. The quantity $A_\tau$ can be pre-computed since it is data-independent.
\end{remark}

%

\begin{remark}[Extension to correlated sensors.] \label{sec:cor_sensor}
The mixture procedure (\ref{mainstatistic}) can be easily extended to the case where sensors are correlated with a known covariance matrix. Define a vector of observations ${\bf y}_{i} = [y_{1, i}, \ldots, y_{N, i}]^\intercal$ for all sensors at time $i$. When there is no change,  ${\bf y}_i$ follows a normal distribution with a  mean vector ${\bf \mu} = [\mu_1, \ldots, \mu_N]^\intercal$ and a covariance matrix $\Sigma_0$. Alternatively, there may exist a change-point at time $\kappa$ such that after the change, the observation vectors are normally distributed with mean vector ${\bf \mu} + (i-\kappa){\bf c}$, ${\bf c} = [c_1, \ldots, c_n]^\intercal$ and the covariance matrix remains $\Sigma_0$ for all $i>\kappa$.
We can whiten the signal vector by ${\bf \widetilde{y}}_i \triangleq \Sigma_0^{-1/2}({\bf y}_i-\mu)$, where $\Sigma_0^{-1/2}$ is the square-root of the positive definite covariance matrix that may be computed via its eigen-decomposition. The coordinates of ${\bf \widetilde{y}}_i$ are independent and the problem then becomes the original hypothesis testing problem (\ref{P1}) with all sensors being affected simultaneously by the change-point, the rate-of-change vector is $\Sigma_0^{-1/2}{\bf c}$, the mean vector is zero before the change, and the covariance remains an identity matrix before and after the change. Hence, after the transform, we may apply the mixture procedure with $p_0 = 1$ on ${\bf \widetilde{y}}_i$.
\end{remark}

\section{Theoretical properties of the detection procedures}\label{sec:theoretical}

In this section we develop theoretical properties of the mixture procedure. We use two standard performance metrics (1) the expected
value of the stopping time when there is no change, the average run length (ARL); (2) the expected detection delay (EDD), defined to be the expected stopping time
in the extreme case where a change occurs immediately at $\kappa = 0$. Since the observations are {\it i.i.d.} under the null,
the EDD provides an upper bound on the expected delay after a change-point until detection occurs when the change occurs later in the sequence of observations (this is also a commonly used fact in change-point detection work \cite{xie2013sequential}). An efficient detection procedure should have a large ARL and meanwhile a small EDD.  Our approximation to the ARL is shown below to be  accurate. In practice, we usually fix ARL to be a large constant, and set the threshold $b$ in (\ref{mainstatistic}) accordingly. The accurate approximation here can be used to find the threshold analytically. Approximation for EDD shows its dependence on a quantity that plays a role of the Kullback-Leibler (KL) divergence, which links to the optimality results in Section \ref{sec:optimality}.

\vspace{.05in}

\noindent {\bf Average run length (ARL).} We present an accurate approximation for ARL of a window limited version of the stopping rule in (\ref{mainstatistic}), which we denote as $\widetilde{T}_2$. Let
\begin{equation}
g(x) \triangleq \log(1-p_0+p_0\exp(x^2/2)),
\label{func}
\end{equation}
and
\[
\psi(\theta)  = \log \mathbb{E}\{\exp[\theta g (Z)  ]\},
\]
where $Z$ has a standard normal distribution. Also let
\[
\gamma(\theta) = \frac{1}{2} \theta^2  \mathbb{E}\left\{  [ \dot{g}_{}(Z) ]^2 \exp \left[ \theta g(Z)  - \psi_{} (\theta) \right] \right\},
\]
and
\[
H(N, \theta)  = \frac{\theta [2\pi  \ddot{\psi}(\theta)]^{1/2}}{\gamma^2(\theta) N^{1/2}}\exp\{N[\theta\dot{\psi}(\theta) - \psi(\theta)]\},
\]
where the dot $\dot{f}$ and double-dot $\ddot{f}$ denote the first-order and second-order derivatives of a function $f$, respectively. Denote by $\phi(x)$ and $\Phi(x)$ the standard normal density function and its distribution function, respectively. Also define a special function $\nu(x) = 2x^{-2}\exp[-2\sum_{n=1}^\infty n^{-1} \Phi(-|x|n^{1/2}/2)]$. For numerical purposes an accurate approximation is given by \cite{Segimundbook2007} 
\[
\nu(x) \approx \frac{(2/x)[\Phi(x/2) - 1/2]}{(x/2)\Phi(x/2) + \phi(x/2)}.
\]

\begin{theorem}[ARL of $\widetilde{T}_2$.]
Assume that $N\rightarrow \infty$ and $b\rightarrow \infty$ with
$b/N$ fixed. Let $\theta$ be defined by $\dot{\psi}(\theta) = b/N$. For a window limited stopping rule of (\ref{mainstatistic}) with $w = o(b^r)$ for some positive integer $r$, we have
\begin{equation}
\begin{split}
\mathbb{E}_\infty\{\widetilde{T}_2\}
= H(N, \theta)\cdot\left[ \int_{\sqrt{2N/(4w/3)^{1/2}}}^{\sqrt{2N/(4/3)^{1/2}}} y \nu^2(y\sqrt{\gamma(\theta)}) dy \right]^{-1} + o(1).
\end{split}
\label{eq:ARL_integration}
\end{equation}
\label{maintheorem}
\end{theorem}
The proof of Theorem \ref{maintheorem} is an extension of the proofs in \cite{xie2013sequential} and \cite{yakir2013extremes} using the change of measure techniques.
To illustrate the accuracy of approximation given in Theorem \ref{maintheorem}, we perform $500$ Monte Carlo trials with $p_0=0.3$, and $w=200$. Figs. \ref{fig:arl}(a) and (b) compare the simulated and theoretical approximation of ARL given in Theorem \ref{maintheorem} when $N = 100$ and $N = 200$, respectively. Note that expression (\ref{eq:ARL_integration}) takes a similar form as the ARL approximation obtained in \cite{xie2013sequential} for the multi-sensor mean-shift case, and only differs in the upper and lower limits in the integration. In Figs. \ref{fig:arl}(a) and (b) we also plot the approximate ARL for the mean shift case in \cite{xie2013sequential}, which shows the importance of having the corrected integration upper and lower limits in our approximation.
In practice, ARL is usually set to $5000$ and $10000$. Table \ref{table:theo} compares the thresholds obtained theoretically and from simulation at these two ARL levels, which demonstrates the accuracy of our approximation. %

\begin{figure}[h]
\begin{center}
\begin{tabular}{cc}
\includegraphics[width = 0.45\linewidth]{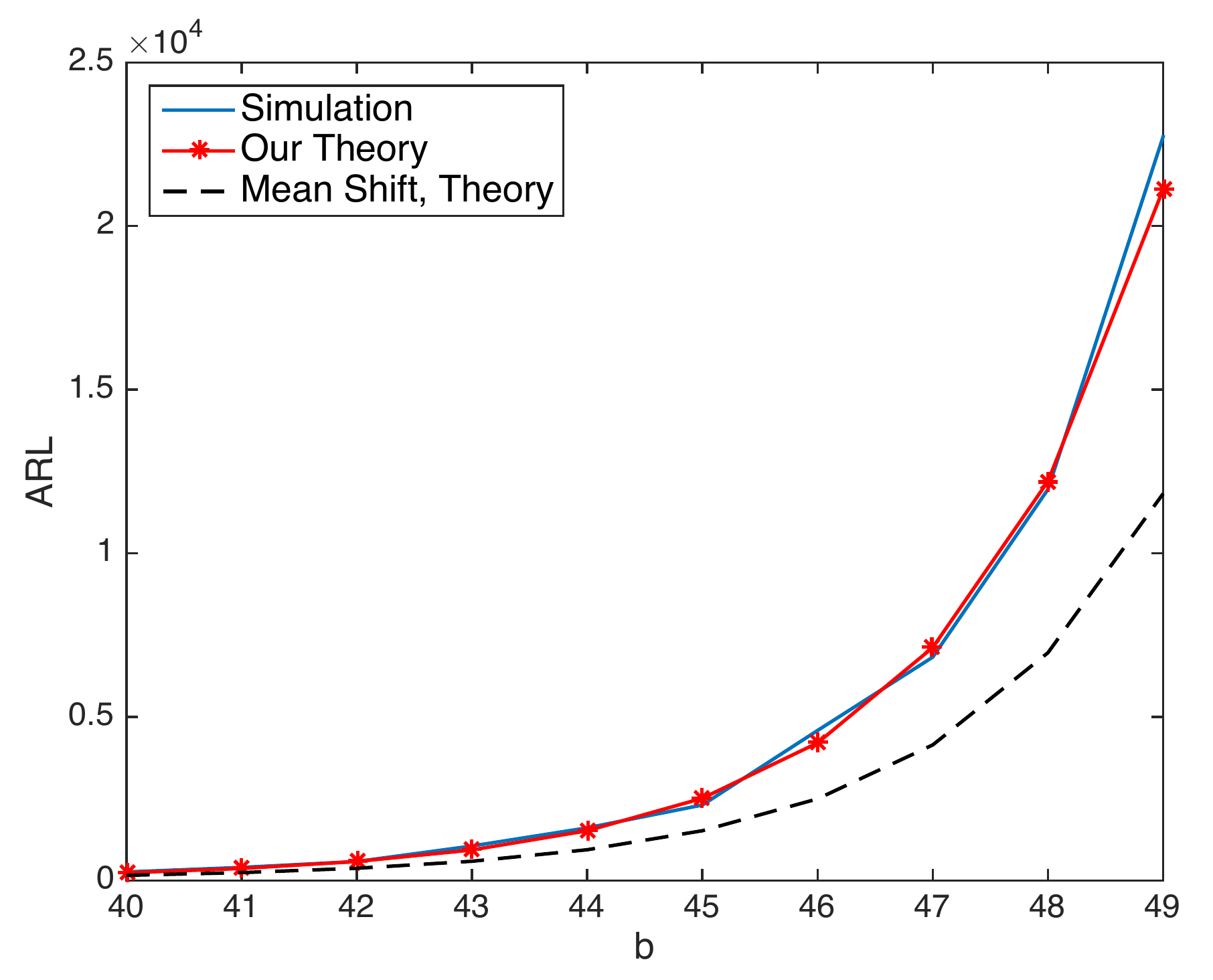} &
\includegraphics[width = 0.45\linewidth]{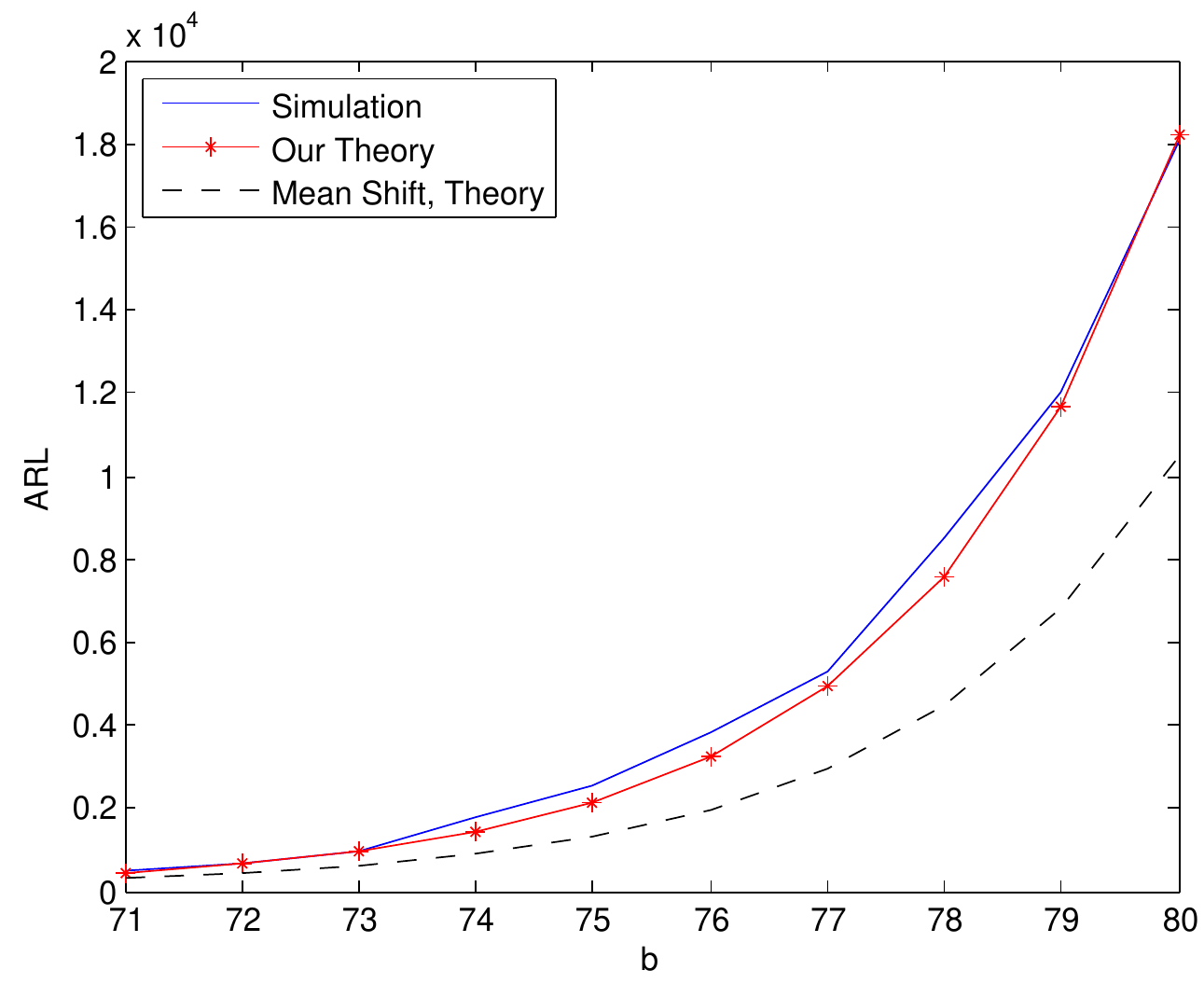}\\
(a) & (b)
\end{tabular}
\caption{(a) Comparison of theoretical and simulated ARL when (a): $N=100$, $p_0=0.3$, and $w=200$; (b): $N=200$, $p_0=0.3$, and $w=200$. }
\label{fig:arl}

\end{center}
\end{figure}

\begin{table}[h!]
\begin{center}
\vspace{-.1in}
\caption{Theoretical versus simulated thresholds for $p_0 = 0.3$, $N=100$ or 200, and $w=200$. 
}
\label{table:theo}
\begin{tabular}{|c|c|c|c|c|}
  \hline
   & ARL & Theory $b$  & Simulated ARL & Simulated $b$\\ \hline
  $N = 100$ & 5000 & 46.34 & 5024 &  46.31 \\ \hline
  &10000 & 47.64 & 10037 &  47.60 \\ \hline
  $N = 200$ & 5000 & 77.04 & 5035 &  76.89 \\ \hline
  & 10000 & 78.66 & 10058 &  78.59 \\ \hline
\end{tabular}
\end{center}
\vspace{-.1in}
\end{table}


\vspace{.1in}
\noindent{\bf Expected detection delay (EDD).} After a change-point occurs, we are interested in the expected number of additional observations required for detection. In this section we establish an approximation upper bound to the expected detection delay. 
Define a quantity
\begin{equation}
\Delta = \left(\sum_{n\in \mathcal{A}} c_n^2/\sigma_n^2 \right)^{1/2},
\label{Delta_def}
\end{equation}
which roughly captures the total signal-to-noise ratio of all affected sensors.
%

\begin{theorem}[EDD of $\widetilde{T}_2$.]\label{thm:EDD}
Suppose $b\rightarrow \infty$, with other parameters held fixed. Let $U$ be a standard normal random variable. If the window length $w$ is sufficiently large and greater than $(6b/\Delta^2)^{1/3}$, then
\begin{equation}
\mathbb{E}_0^\mathcal{A}\{\widetilde{T}_2\} \leq \left \{\frac{b-|\mathcal{A}|\log p_0 - (N-|\mathcal{A}|)\mathbb{E}\{g(U)\}}{\Delta^2/6}\right\}^{1/3} + o(1),
\label{EDD_expr}
\end{equation}
\end{theorem}
where $\mathbb{E}_0^\mathcal{A}$ is defined at the beginning of Section \ref{sec:formulation}.
To demonstrate the accuracy of (\ref{EDD_expr}), we perform 500 Monte Carlo trials. In each trial, we let the change-point happen at the initial time and randomly select $Np$ sensors affected by the change and set the rate-of-change $c_n = c$ for a constant $c$, $n \in \mathcal{A}$. The thresholds for each procedure are set so that their ARLs are equal to 5000. Fig. \ref{fig:simulationDD} shows EDD versus $c$, where our upper bound turns out to be an accurate approximation to EDD.

\begin{figure}[h]
\begin{center}
\includegraphics[width = 0.45\linewidth]{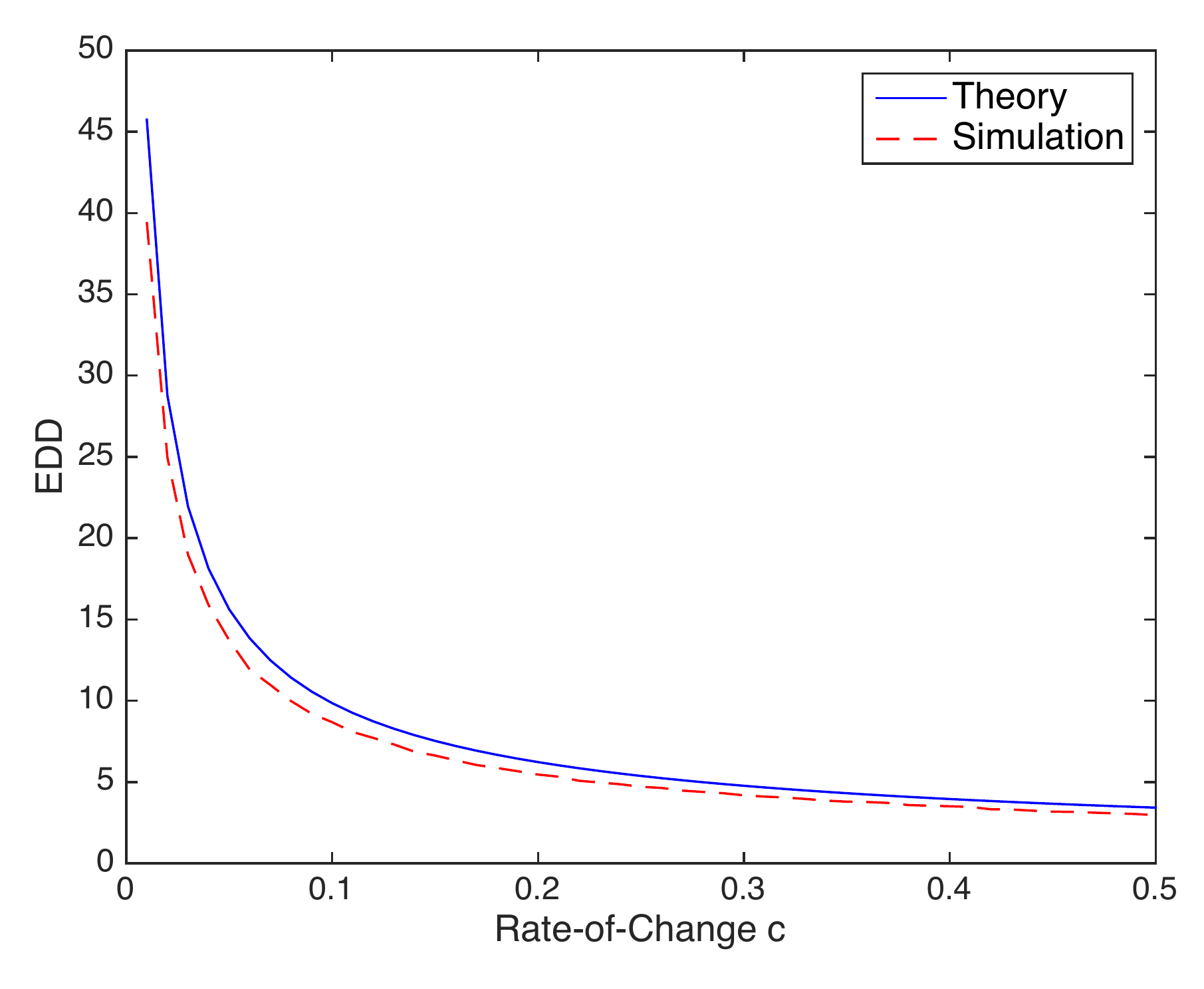}
\caption{Comparison of theoretical and simulated EDD when $N=100$, $p_0=0.3$, $p=0.3$, and  $w=200$. All rate-of-change $c_n = c$ for affected sensors. }
\label{fig:simulationDD}
\end{center}
\vspace{-0.1in}
\end{figure}

\section{Optimality}\label{sec:optimality}

In this section, we prove that our detection procedures: $T_1$ and the window limited versions $\widetilde{T}_1$ and $\widetilde{T}_2$ are asymptotically first order optimal. The optimality proofs here extends the results in \cite{tartakovsky2014sequential}, \cite{lai1998information}, for our multi-sensor {\it non-i.i.d.} data setting. The {\it non-i.i.d.}ness is due to the fact that under the alternative, the means of the samples change linearly as the number of post-change samples grows. Following the classic setup, we consider a class of detection procedures with their ARL greater than some constant $\gamma$, and then find an optimal procedure within such a class to minimize the detection delay. Since it is difficult to establish an uniformly optimal procedure for any given $\gamma$, we consider the asymptotic optimality when $\gamma$ tends to infinity.

We first study a general setup with non-$i.i.d.$ distributions for the multi-sensor problem, and establish optimality of two general procedures related to $T_1$ and $T_2$. Then we specialize the results to the multi-sensor slope-change detection problem. In particular, we generalize the lower bound for the detection delay from the single sensor case (Theorem 8.2.2 in \cite{tartakovsky2014sequential} and Theorem 1 in \cite{lai1998information}) to our multi-sensor case. We also
generalize the result therein to our setting where the log-likelihood ratio grows polynomially on the order of $j^q$ for $q\geq 1$ as the number of post-change observations $j$ grows (in the classic setting $q = 1$); this is used to account for the non-stationarity in our problem. 

\vspace{.1in}
\noindent{\bf Setup for general non-$i.i.d.$ case.} Consider a setup for the multi-sensor problem with non-$i.i.d.$ data. Assume there are $N$ sensors that are independent (or with known covariance matrix so the observations can be whitened across sensors), and that the change-point affects all  sensors simultaneously.
Observations at the $n$th sensor are denoted by $x_{n,t}$ over time $t=1,2,\ldots$. If there is no change, $x_{n,t}$ are distributed according to conditional densities $f_{n,t}(x_{n,t}|x_{n,[1,t-1]})$, where $x_{n,[1,t-1]}$ = $(x_{n,1},\ldots,x_{n,t-1})$ (this allows the distributions at time $t$ to be dependent on the previous observations).
Alternatively, if a change-point occurs at time $\kappa$ and the $n$th sensor is affected, $x_{n,t}$ are distributed according to conditional densities $f_{n,t}(x_{n,t}|x_{n,[1,t-1]})$ for $t = 1,\ldots, \kappa$,  and are according to $g_{n,t}^{(\kappa)}(x_{n,t}|x_{n,[1,t-1]})$ for $t  > \kappa$. Note that the post-change densities are allowed to be dependent on the change-point $\kappa$. Define a filtration at time $t$ by $\mathcal{F}_{t}=\sigma(x_{1,[1,t]},\ldots,x_{N,[1,t]})$. Again, assume a subset $\mathcal{A} \subseteq \{1,2,\ldots,N\}$ of sensors are affected by the change-point.  Similar to Section \ref{sec:formulation}, with a slight abuse of notation, we denote $\mathbb{P}_\infty$, $\mathbb{E}_\infty$, $\mathbb{P}_{\kappa}^{\mathcal{A}}$ and $\mathbb{E}_{\kappa}^{\mathcal{A}}$ as the probability and expectation when there is no change, or when a change occurs at time $\kappa$ and a set $\mathcal{A}$ of sensors are affected by the change, with the understanding that here the probability measures are defined using the conditional densities.

\vspace{.1in}
\noindent {\bf Optimality criteria.} We adopt two commonly used minimax criteria to establish the optimality of a detection procedure $T$.  Similar to Chapter 8.2.5 of \cite{tartakovsky2014sequential}, we consider two criterions associated with the $m$-th moment of the  detection delay for $m\geq 1$. The first criterion is motivated by Lorden's work  \cite{lorden1971procedures}, which minimizes the worst-case delay
\begin{equation}
\mbox{ESM}_m^{\mathcal{A}}(T) \triangleq \sup_{0\leq k < \infty} \mbox{esssup} ~ \mathbb{E}_{k}^{\mathcal{A}}\left\{ [(T-k)^+]^m|\mathcal{F}_{k} \right\},\label{detectiondelays}
\end{equation}
where ``esssup'' denotes the measure theoretic supremum that excluded points of measure zero. In other words, the definition (\ref{detectiondelays})  first maximizes over all possible trajectories of observations up to the change-point and then over the change-point time. The second criterion is motivated by Pollak's work  \cite{pollak1985optimal}, which minimizes the maximal conditional average detection delay
\begin{equation}
\mbox{SM}_m^{\mathcal{A}}(T) \triangleq \sup_{0\leq k < \infty} \mathbb{E}_{k}^{\mathcal{A}} \left\{ (T-k)^m | T>k\right\}.
\label{detectiondelays2}
\end{equation}
The extended Pollak's criterion  (\ref{detectiondelays2}) is not as strict as the extended Lorden's criterion in the sense that SM$_m^\mathcal{A}(T) \leq \mbox{ESM}_m^\mathcal{A}(T)$, and  we prefer (\ref{detectiondelays2}) since  it is connected to the conventional decision theoretic approach and the resulted optimization problem can possibly be solved by a least favorable prior approach.
The EDD defined earlier in Section \ref{sec:theoretical} can be viewed as ESM$_m$ and SM$_m$ for $m = 1$, and the supremum over $k$ happens when $k = 0$.

%
%
Define $C(\gamma)$ to be a class of detection procedures with their ARL  greater than $\gamma$: 
\[C(\gamma) \triangleq \{T:\mathbb{E}_{\infty}\{T\}\geq \gamma\}.\]
A procedure $T$ is optimal, if it belongs to $C(\gamma)$ and minimizes  $\mbox{ESM}_m(T)$ or $\mbox{SM}_m(T)$.

\vspace{0.1in}
\noindent
{\bf Optimality for general non-$i.i.d$ setup.}
Under the above assumptions, the log-likelihood ratio for each sensor is given by
\[
\lambda_{n,k,t} = \sum_{i=k+1}^{t}   \log \frac{g_{n,i}^{(k)}(x_{n,i}|x_{n,[1,i-1]})}{f_{n,i}(x_{n,i}|x_{n,[1,i-1]})}.
\]
For any set $\mathcal{A}$ of affected sensors, the
log-likelihood ratio is given by
\begin{equation}
\lambda_{\mathcal{A},k,t} =
\sum_{n \in \mathcal{A}}\lambda_{n,k,t}.
\label{lambda_A}
\end{equation}

We first establish an lower bound for any detection procedure.
The constant $I_{\mathcal{A}}$ below can be understood intuitively as a surrogate for the Kullback-Leibler (KL) divergence in the hypothesis problem. When the observations are $i.i.d.$, $I_\mathcal{A}$ is precisely the KL divergence \cite{lai1998information}.

\begin{theorem}[General lower bound.]\label{lowerbound_multi}
For any $\mathcal{A} \subseteq \{1, \ldots, N\}$ such that there exists  some $q\geq 1$, $j^{-q} \lambda_{\mathcal{A},k,k+j}$  converges in probability to a positive constant $I_{\mathcal{A}} \in (0, \infty)$ under $\mathbb{P}_{k}^{\mathcal{A}}$,
\begin{equation}
\frac{1}{j^q} \lambda_{\mathcal{A},k,k+j} \xrightarrow[j\rightarrow \infty]{\mathbb{P}_{k}^{\mathcal{A}}} I_{\mathcal{A}},
\end{equation}
and in addition, for all $\varepsilon >0$, for an arbitrary $M \rightarrow \infty$
\begin{equation}
\sup_{0\leq k < \infty} \mbox{esssup}~ \mathbb{P}_{k}^{\mathcal{A}} \left\{M^{-q} \max_{0\leq j<M} \lambda_{\mathcal{A},k,k+j} \geq (1+\varepsilon)I_{\mathcal{A}} \middle| \mathcal{F}_{k}\right\} \xrightarrow[M\rightarrow \infty]{} 0.
\label{secondassumption_multi}
\end{equation}
Then,
\begin{itemize}
\item[(i)] for all $0<\varepsilon <1$, there exists some $k\geq 0$ such that
\begin{equation}
\lim_{\gamma\rightarrow \infty} \sup_{T\in C(\gamma)} \mathbb{P}_{k}^{\mathcal{A}}\left\{ k<T<k+(1-\varepsilon)(I_{\mathcal{A}}^{-1}\log \gamma)^{\frac{1}{q}} \middle|T>k\right\} =0.
\label{firstresult_multi}
\end{equation}
\item[(ii)] for all $m\geq 1$,
\begin{equation}
\liminf_{\gamma \rightarrow \infty} \frac{\inf_{T\in C(\gamma)} \mbox{ESM}_m^{\mathcal{A}}(T)}{(\log \gamma)^{m/q}} \geq \liminf_{\gamma \rightarrow \infty} \frac{\inf_{T\in C(\gamma)} \mbox{SM}_m^{\mathcal{A}}(T)}{(\log \gamma)^{m/q}}\geq \frac{1}{I_{\mathcal{A}}^{m/q}}.
\label{ESM_multi}
\end{equation}
\end{itemize}
\end{theorem}

Consider a general mixture CUSUM procedure related to $T_1$, which has also been studied in  \cite{siegmund2011detecting} and  \cite{xie2013sequential}:
\begin{equation}
T_{\rm CS} = \inf \left\{t:  \max_{0\leq k<t} \sum_{n=1}^N\log(1-p_0+p_0\exp(\lambda_{n,k,t})) \geq b\right\},
\label{mixtureCUSUM}
\end{equation}
where $b$ is a prescribed threshold. The following lemma shows that for an appropriate choice of the threshold $b$, $T_{\rm CS}$ has an ARL lower bounded by  $\gamma$ and, hence, for such thresholds it belongs to $C(\gamma)$.
\begin{lemma}
\label{ARL2FA_mixtureCUSUM}
For any $p_0\in (0,1]$, $T_{\rm CS}(b) \in C(\gamma)$, provided $b\geq \log \gamma$.
\end{lemma}
\begin{theorem}[Optimality of $T_{\rm CS}$.]\label{optimality_mixture_CUSUM}
For any $\mathcal{A}\subseteq \{1, \ldots, N\}$ such that there exists some $q\geq 1$ and a finite positive number $I_{\mathcal{A}} \in (0, \infty)$ for which (\ref{secondassumption_multi}) holds, and for all $\varepsilon \in (0,1)$ and $t\geq 0$,
\begin{equation}
\sup_{0\leq k < t} \mbox{esssup}~ \mathbb{P}_{k}^{\mathcal{A}}\left( j^{-q} \lambda_{\mathcal{A},k,k+j} < I_{\mathcal{A}}(1-\varepsilon)\middle| \mathcal{F}_{k}\right) \xrightarrow[j\rightarrow \infty]{} 0.
\label{assumption_optimality_mixture_CUSUM}
\end{equation}
If $b\geq \log \gamma$ and $b = \mathcal{O}(\log \gamma)$, then $T_{\rm CS}$ is asymptotically minimax in the class $C(\gamma)$ in the sense of minimizing $\mbox{ESM}_m^{\mathcal{A}}(T)$ and $\mbox{SM}_m^{\mathcal{A}}(T)$ for all $m\geq 1$ to the first order as $\gamma \xrightarrow[]{} \infty$.
\end{theorem}
%

We can also prove that the window-limited version $\widetilde{T}_{\rm CS}$ is asymptotically optimal. Since the window length affects ARL and the detection delay, in the following we denote this dependence more explicitly by $w_\gamma$.
\begin{corollary}[Optimality of $\widetilde{T}_{\rm CS}$.]\label{cor_TCS}
Assume the conditions in Theorem \ref{optimality_mixture_CUSUM} hold and in addition,
\begin{equation}
\liminf_{\gamma \rightarrow \infty} \frac{w_{\gamma}}{\left(\log \gamma/I_{\mathcal{A}}\right)^{1/q}} >1.
\label{window_mixture_assumption}
\end{equation}
If $b\geq \log \gamma$ and $b = \mathcal{O}(\log \gamma)$,
then $\widetilde{T}_{\rm CS}(b)$ is asymptotically minimax in the class $C(\gamma)$ in the sense of minimizing $\mbox{ESM}_m^{\mathcal{A}}(T)$ and $\mbox{SM}_m^{\mathcal{A}}(T)$ for all $m\geq 1$ to the first order as $\gamma \xrightarrow[]{} \infty$.
\label{optimality_window_mixture_CUSUM}
\end{corollary}
Intuitively, this means that the window length should be greater than the first order approximation to the detection delay $[(\log \gamma) / I_\mathcal{A}]^{1/q}$. Note that our earlier result (\ref{EDD_expr}) for the expected detection delay of the multi-sensor case is of this form for $q = 3$ and $I_\mathcal{A} =  \Delta^2/6$.

Similarly,  we may consider a general mixture GLR procedure related to $T_2$ as in \cite{xie2013sequential}. Denote the log-likelihood (\ref{lambda_A}) as $\lambda_{\mathcal{A}, k,t}(\theta)$ to emphasize its  dependence on an unknown parameter $\theta$. The mixture GLR procedure maximizes $\theta$ over a parameter space $\Theta$ before combining them across all sensors. 
%
Unfortunately, we are unable to establish the asymptotic optimality for the general GLR procedure 
and its window limited version, due to a lack of martingale property.

\vspace{.1in}
\noindent
{\bf Optimality for multi-sensor slope change.}
Note that $T_1$ and $\widetilde{T}_1$ correspond to special cases of $T_{\rm CS}$, $\widetilde{T}_{\rm CS}$, so we can use Theorem \ref{optimality_mixture_CUSUM} and Corollary \ref{cor_TCS} to show their optimality by checking conditions. Although we are not able to establish optimality of the general mixture GLR procedure as mentioned above, we can prove the optimality for  $\widetilde{T}_2$ by exploiting the structure of the problem.


\begin{lemma}[Lower bound.]
For the multi-sensor slope change detection problem in (\ref{P1}), for a non-empty set $\mathcal{A} \subseteq \{1,\ldots,N\}$, the conditions of Theorem \ref{lowerbound_multi} are satisfied when $q=3$ and $I_{\mathcal{A}} = \Delta^2/6$.
\label{checkassuption}
\end{lemma}
The following lemma plays a similar role as the general version Lemma \ref{ARL2FA_mixtureCUSUM} in our multi-sensor case in (\ref{P1}), and it
shows that for a properly chosen threshold $b$, ARL of $\widetilde{T}_2$ is lower bounded by $\gamma$ and, hence, for such threshold it belongs to $C(\gamma)$.
\begin{lemma}
For any $p_0\in (0,1]$, $\widetilde{T}_2(b) \in C(\gamma)$, provided
\[
b \geq N/2-4\log \left[ 1-\left(1-1/\gamma\right)^{1/w_{\gamma}} \right].
\]
\label{ARL2FA_multi_GCUSUM}
\end{lemma}

\begin{remark}[Implication on window length.]
Lemma \ref{ARL2FA_multi_GCUSUM} shows that to have $b = O(\log \gamma)$, we need $\log w_{\gamma} = o(\log \gamma)$.
\end{remark}

\begin{theorem}[Asymptotical optimality of $T_1$, $\widetilde{T}_1$ and $\widetilde{T}_2$.]
Consider the multi-sensor slope change detection problem (\ref{P1}).   
\begin{itemize}
\item[(i)] If $b\geq \log \gamma$ and $b=\mathcal{O}(\log \gamma)$, then $T_1(b)$ is asymptotically minimax in class $C(\gamma)$ in the sense of minimizing expected moments $\mbox{ESM}_m^{\mathcal{A}}(T)$ and $\mbox{SM}_m^{\mathcal{A}}(T)$ for all $m\geq 1$ to the first order as $\gamma \xrightarrow[]{} \infty$.

\item[(ii)] In addition to conditions in (i), if the window length satisfies
\begin{equation}
\liminf_{\gamma \rightarrow \infty} \frac{w_{\gamma}}{\left[6 (\log \gamma)/ \Delta^2\right]^{1/3}} >1,
\label{window_multi_sensor}
\end{equation}
then $\widetilde{T}_1(b)$  is asymptotically minimax in class $C(\gamma)$ in the sense of minimizing expected moments $\mbox{ESM}_m^{\mathcal{A}}(T)$ and $\mbox{SM}_m^{\mathcal{A}}(T)$ for all $m\geq 1$ to first order as $\gamma \xrightarrow[]{} \infty$.

\item[(iii)] If $b \geq N/2-4\log [ 1-\left(1-1/\gamma\right)^{1/w_{\gamma}} ]$,  $b=\mathcal{O}(\log \gamma)$, the window length satisfies $\log(w_{\gamma}) = o(\log \gamma)$ and (\ref{window_multi_sensor}) holds,
then $\widetilde{T}_2(b)$ is asymptotically minimax in class $C(\gamma)$ in the sense of minimizing $\mbox{ESM}_m^{\mathcal{A}}(T)$ and $\mbox{SM}_m^{\mathcal{A}}(T)$ for $m = 1$ to first order as $\gamma \xrightarrow[]{} \infty$.
\end{itemize}
\label{optimality_here}
\end{theorem}

\begin{remark}
Above we prove the optimality of $T_1(b)$ and $\widetilde{T}_1(b)$  for $m\geq 1$. However, we can only prove the optimality of $\widetilde{T}_2(b)$ for a special case $m = 1$, due to a lack of martingale properties here.
\end{remark}

\section{Numerical Examples}\label{sec:num_eg}

\noindent{\bf Comparison with mean-shift GLR procedures.} We compare the mixture procedure for slope change detection, with the classic multivariate CUSUM \cite{multivariate_CUSUM} and the mixture procedure for mean shift detection \cite{xie2013sequential}. The multivariate CUSUM essentially forms a CUSUM statistic at each sensor, and raises an alarm whenever a single sensor statistic hits the threshold. As commented earlier in Remark \ref{remark_mean_shift}, the only difference between $\widetilde{T}_2$ and the mixture procedure for mean shift in \cite{xie2013sequential} is how $U_{n, k, t}$ is defined. Following the steps for deriving (\ref{firstorderEDD}), we can show that the mean shift mixture procedure is also asymptotically optimal for the slope change detection problem. Here, our numerical example verifies this, and show that the improvement of EDD by using $\widetilde{T}_2$ versus the multi-variate CUSUM and the mean-shift mixture procedure is not significant. However, the mean-shift mixture procedure fails to estimate the change-point time accurately due to model mismatch. Fig. \ref{fig:compare_MSGLR_SLGLR} shows the mean square error for estimating the change-point time $\kappa$, using the multi-chart CUSUM, the mean-shift mixture procedure, and $\widetilde{T}_2$, respectively. Note that $\widetilde{T}_2$ has a significant improvement.

\begin{figure}[h]
\begin{center}
\includegraphics[width = 0.5\linewidth]{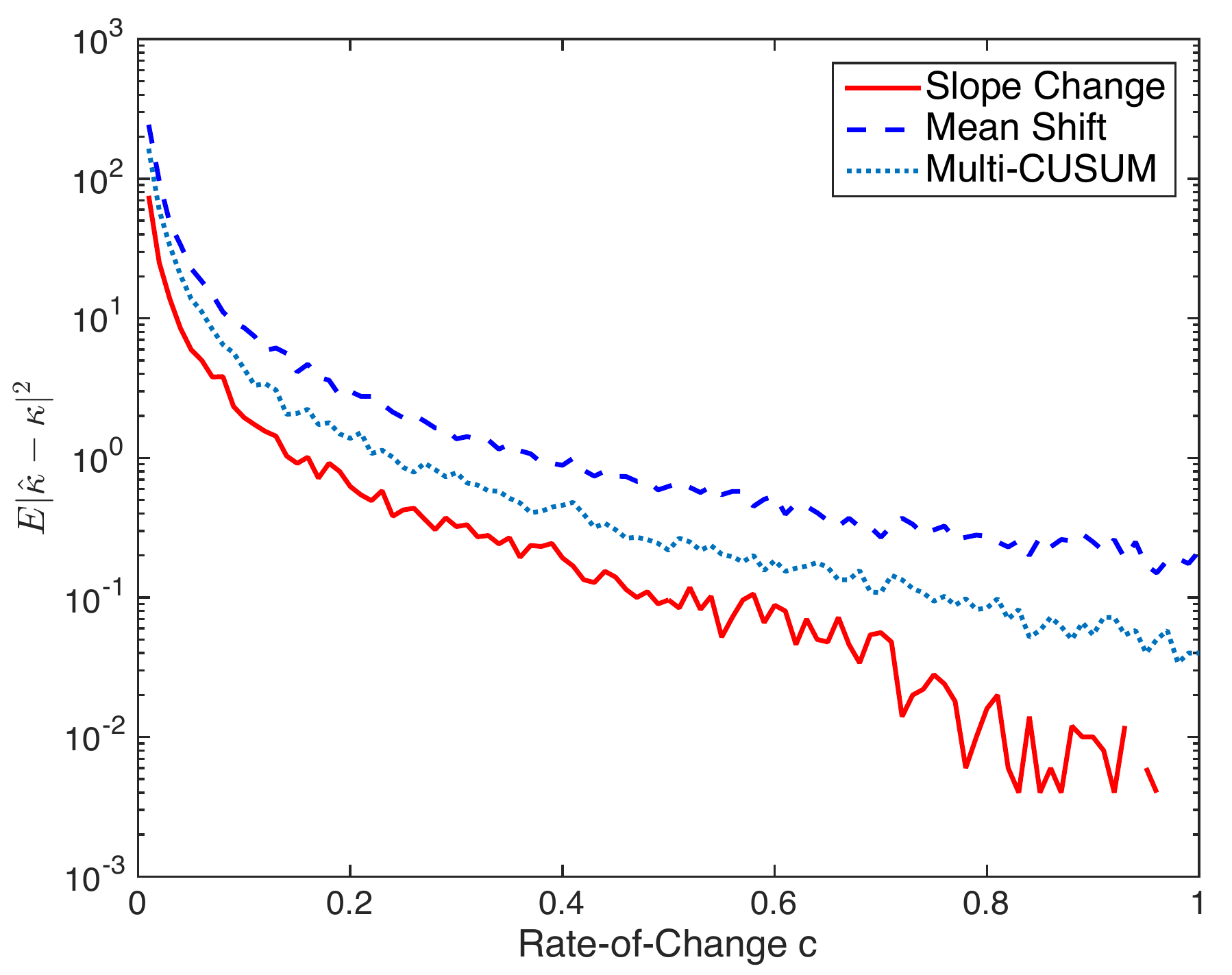}
\caption{Comparison of mean square error for estimating the change-point time for the mixture procedure tailored to slope change $\widetilde{T}_2$, the mixture procedure with mean shift, and multi-chart CUSUM, when $N=100$, $p_0=0.3$, $p=0.5$ and $w=200$.}
\label{fig:compare_MSGLR_SLGLR}
\end{center}
\vspace{-.1in}
\end{figure}

\vspace{.1in}
\noindent
{\bf Financial time series.}
In the earlier example illustrated in Fig. \ref{fig:finance}(a), the goal is to detect a trend change online. Clearly a change-point occurs at time 8000 in the stock price, and such a change-point is verifiable. Fig. \ref{fig:finance}(b) shows that there is a peak in the bid size versus the ask size, which usually indicates a change in the trend of the price (possible with some delay). To illustrate the performance of our method in this financial dataset, we plot the detection statistics by using a ``single-sensor'', i.e., using only one data stream, and by using ``multi-sensor'' scheme, i.e. using data from multiple streams, which in this case correspond to $8$ factors (e.g, stock price, total volume, bid size and bid price, as well ask size and ask price). In fact, only 4 factors out of 8 factors contain the change-point. Fig. \ref{fig:finance}(c) plots the statistic if we use only a single-sensor. Fig. \ref{fig:finance}(d)  illustrates the statistic when we use all the 8 factors and preprocess by whitening with the covariance of the factors as described in Section \ref{sec:cor_sensor}. The statistics all rise around time 8000 with the multi-sensor statistic to be smoother and indicates a lower false detection rate. Looking at Fig. \ref{fig:finance}(d), after the major trend change (around sample index 8000), the multi-chart CUSUM statistic rises the slowest. Although it appears, the slope-change mixture procedure rises a bit slower than the mean-shift mixture procedure, we demonstrate in simulation that for fixed ARL these two procedures have similar EDD¡¯s, and also in Fig. 5 that the slope-change mixture procedure has a better performance in estimating $k^*$ than the mean-shift mixture procedure. Therefore, the slope-change mixture procedure is still preferrable.


\begin{figure}[h!]
\begin{center}
\begin{tabular} {cc}
\includegraphics[width=4cm,height=3cm]{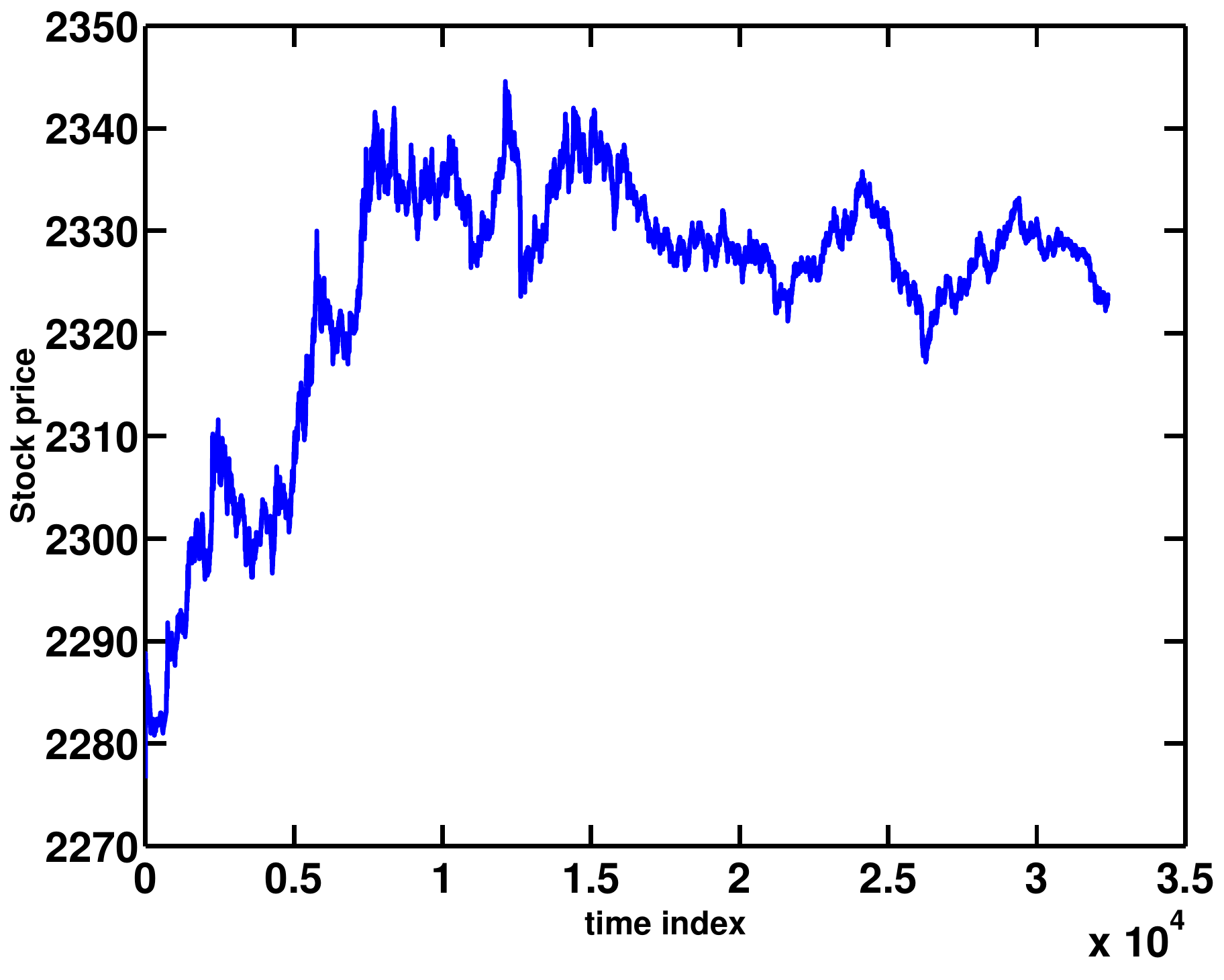} &
\includegraphics[width=4cm,height=3cm]{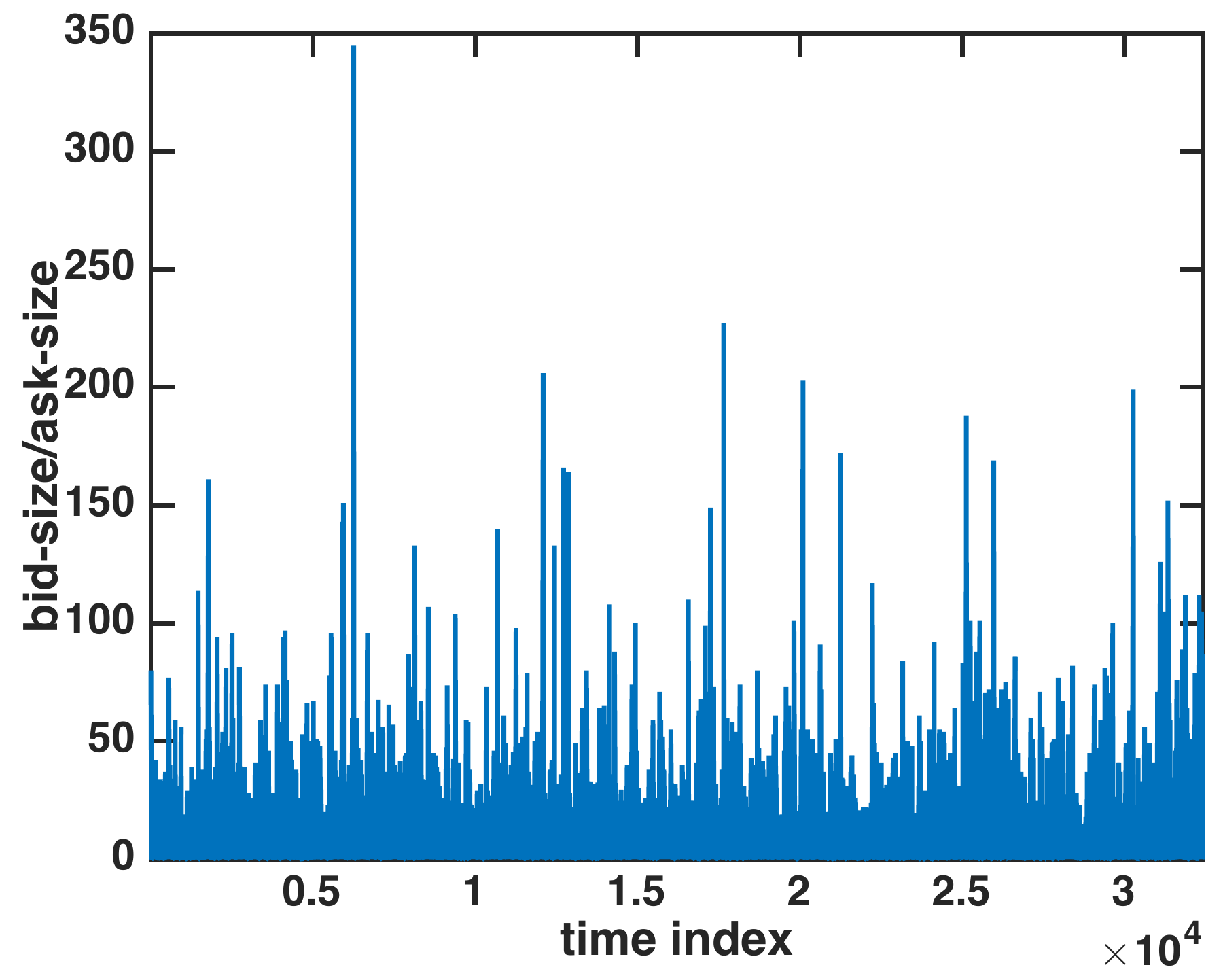} \\
(a) Stock price data. & (b) ask-size/bid-size \\
 \includegraphics[width=4cm,height=3cm]{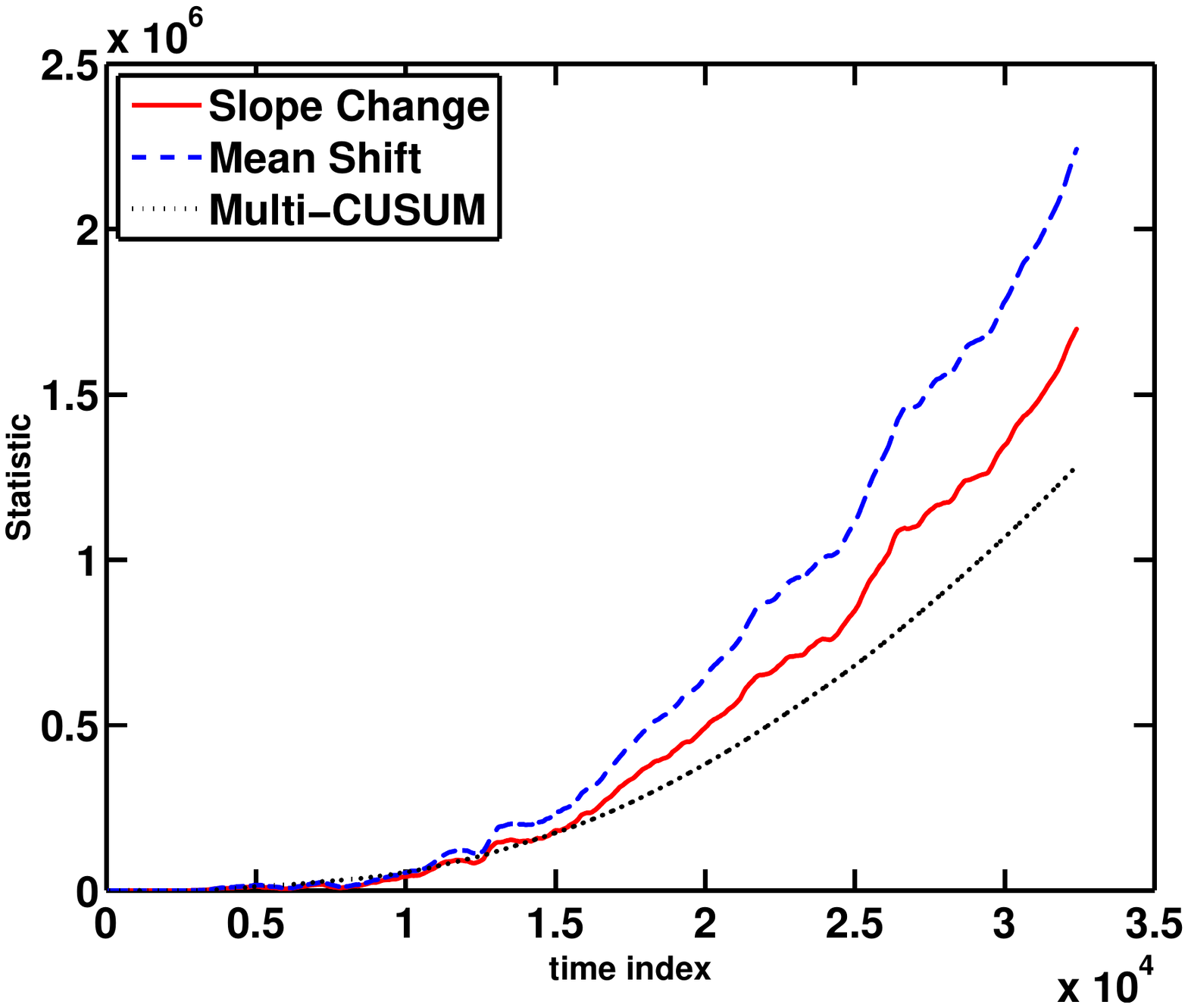} &
 \includegraphics[width=4cm,height=3cm]{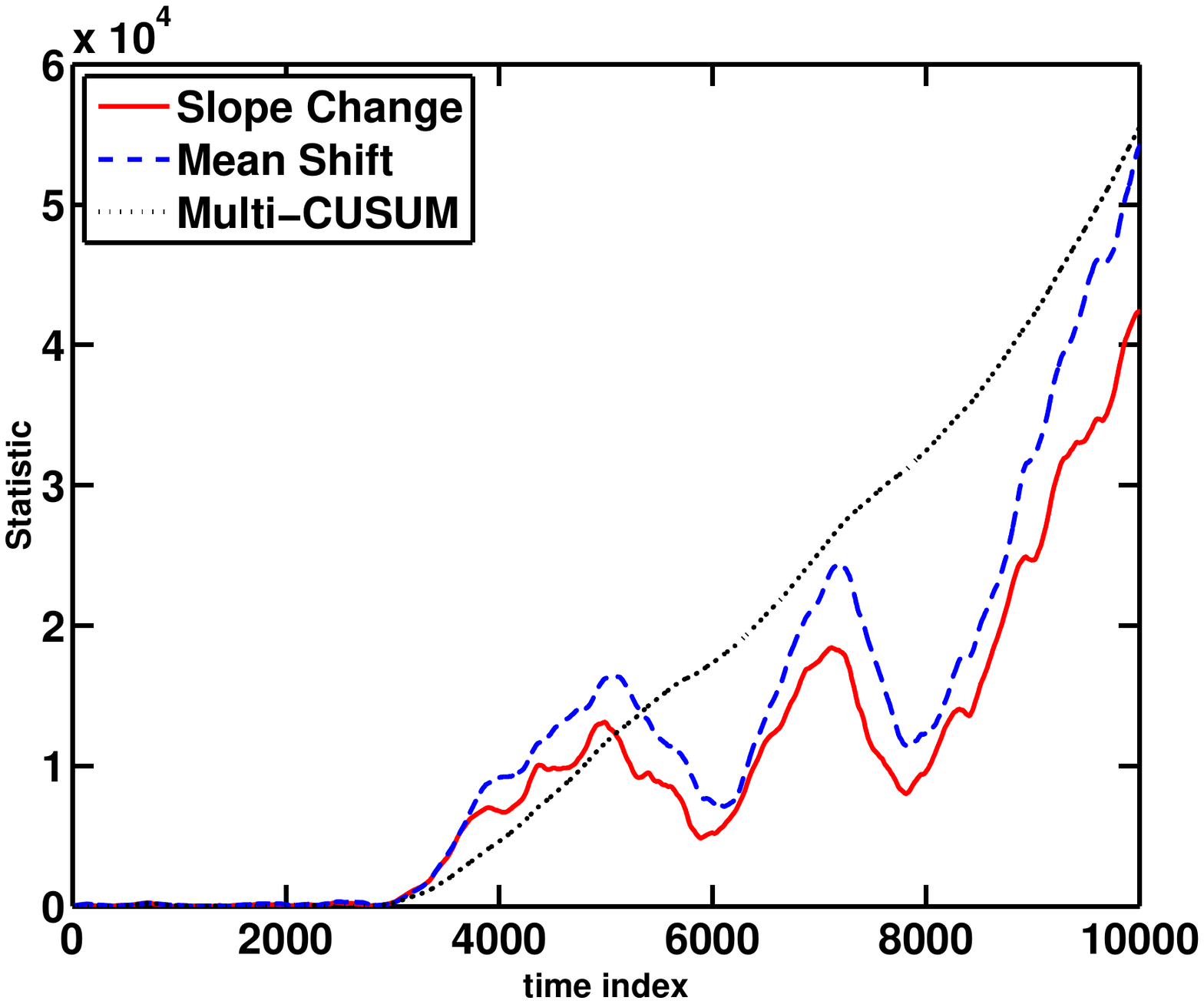}\\
(c) Single-sensor. & (d) Zoom-in of (c). \\
 \includegraphics[width=4cm,height=3cm]{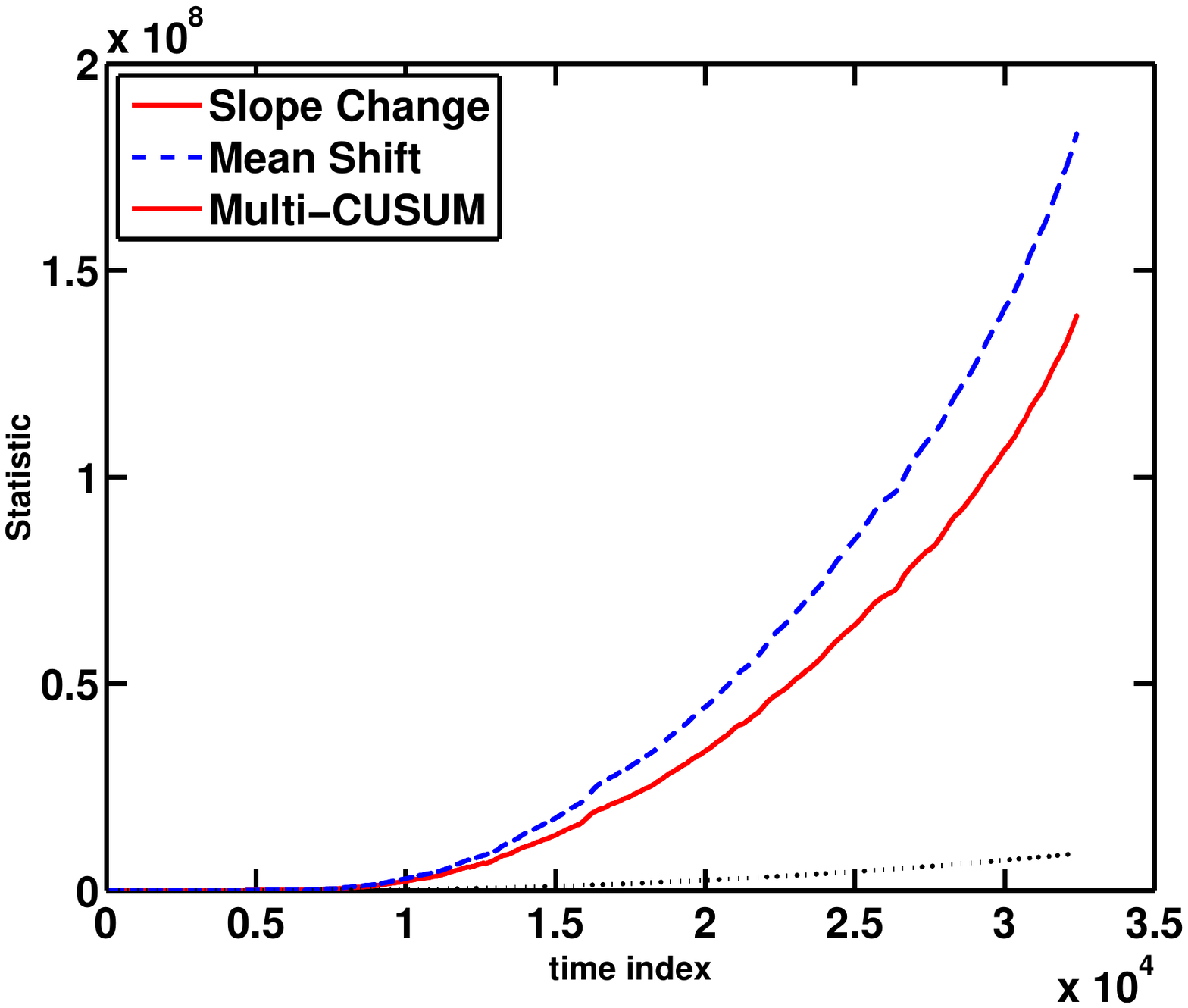}  &  \includegraphics[width=4cm,height=3cm]{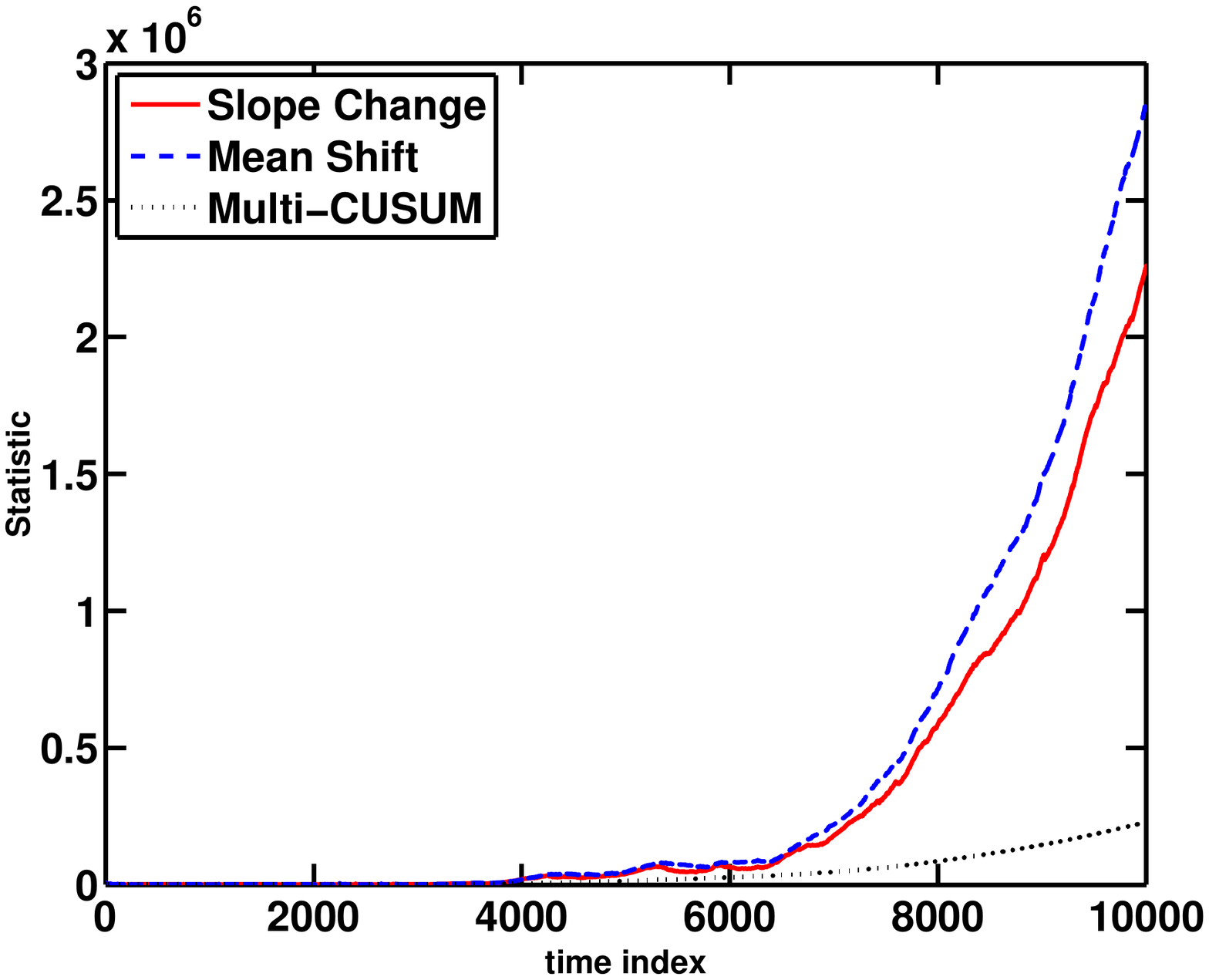}\\
(e) Multi-sensor, correlation. & (f) Zoom-in of (e).
\end{tabular}
\end{center}
\caption{Statistic for detecting trend changes in financial time series with $w=500$ for both single sensor and multi-sensor procedures, and $p_0=1$ for the multi-sensor procedure.}
\label{fig:finance}
\end{figure}

\vspace{.1in}
\noindent{\bf Aircraft engine multi-sensor prognostic.}
We present an engine prognostic example using the aircraft turbofan engine dataset simulated by NASA\footnote{Data can be downloaded from http://ti.arc.nasa.gov/tech/dash/pcoe/prognostic-data-repository/}. In the dataset, multiple sensors measure different physical properties of the aircraft engine to detect a faulty condition and to predict the whole life time. The dataset contains $100$ training systems and $100$ testing systems. Each system is monitored by $N = 21$ sensors. In the training dataset, we have a complete sample path from the initial time to failure for each of the $21$ sensors of each training system. In the testing dataset, we only have partial sample paths (i.e., the system fails eventually but we have not observed that yet and it still has a remaining life). Our goal is to predict the whole life for the test systems using available observations. The dataset also provides ground truth, i.e., the actual failure times (or equivalently the whole life) of the testing systems.

We first apply our mixture procedures to each training system $j$, $j = 1, \ldots, 100$, to estimate a change-point location $\kappa_j$ (which corresponds to the maximizer of $k$ in the definition of $\widetilde{T}_2$ when the procedure stops),  and the rate-of-change at $n$th sensor for  the $j$th system $\hat{c}_{n, j}$ using (\ref{estimation_of_slope}). Then fit a simple survival model using $\hat{\kappa}_j$ and $\hat{c}_{n, j}$ as regressors in determining the remaining life. We build a model for the Time-To-Failure (TTF) $Y_j$ of system $j$ based a log location-normal model, which is commonly used in reliability theory \cite{fang2015adaptive}:
$
    \mathbb{P}\{Y_j \leq y\} = \Phi\left[(\log(y)-\pi_j)/\eta\right],
$
where $\eta$ is a user specified scale parameter that is assumed to be the same for each system, $\pi_j$ is the location parameter that is assumed to be a linear function of the rate-of-change:
$
\pi_j = \beta_0 + \sum_{n=1}^N \beta_{j}\hat{c}_{n,j},
$
where $(\beta_0, \beta_1, \ldots, \beta_N)$ is a vector of the regression coefficients that are estimated by maximum likelihood.
Next, we apply the mixture procedure on the $j$th testing system to estimate the change-point time $\hat{\kappa}_j$ and the rate-of-change $\hat{c}_{n, j}$, and substitute them into the fitted models to determine a TTF using the mean value. The whole life of the $j$th system is estimated as $\hat{\kappa}_j$ plus its mean TTF.

We use the relative prediction error as performance metric, which is the absolute difference between the estimated life and the actual whole life, divided by the actual whole life. Fig. \ref{fig:system1} shows the box-plot of the relative prediction error versus threshold $b$. Our method based on change-point detection works well and it has a mean relative prediction error  around 10\%. Here the choice of the threshold $b$ has a tradeoff: the relative prediction error decreases with a larger $b$; however, a larger $b$ also causes a longer detection delay.

\begin{figure}[h]
\begin{center}
\begin{tabular}{cc}
\includegraphics[width = .5\linewidth]{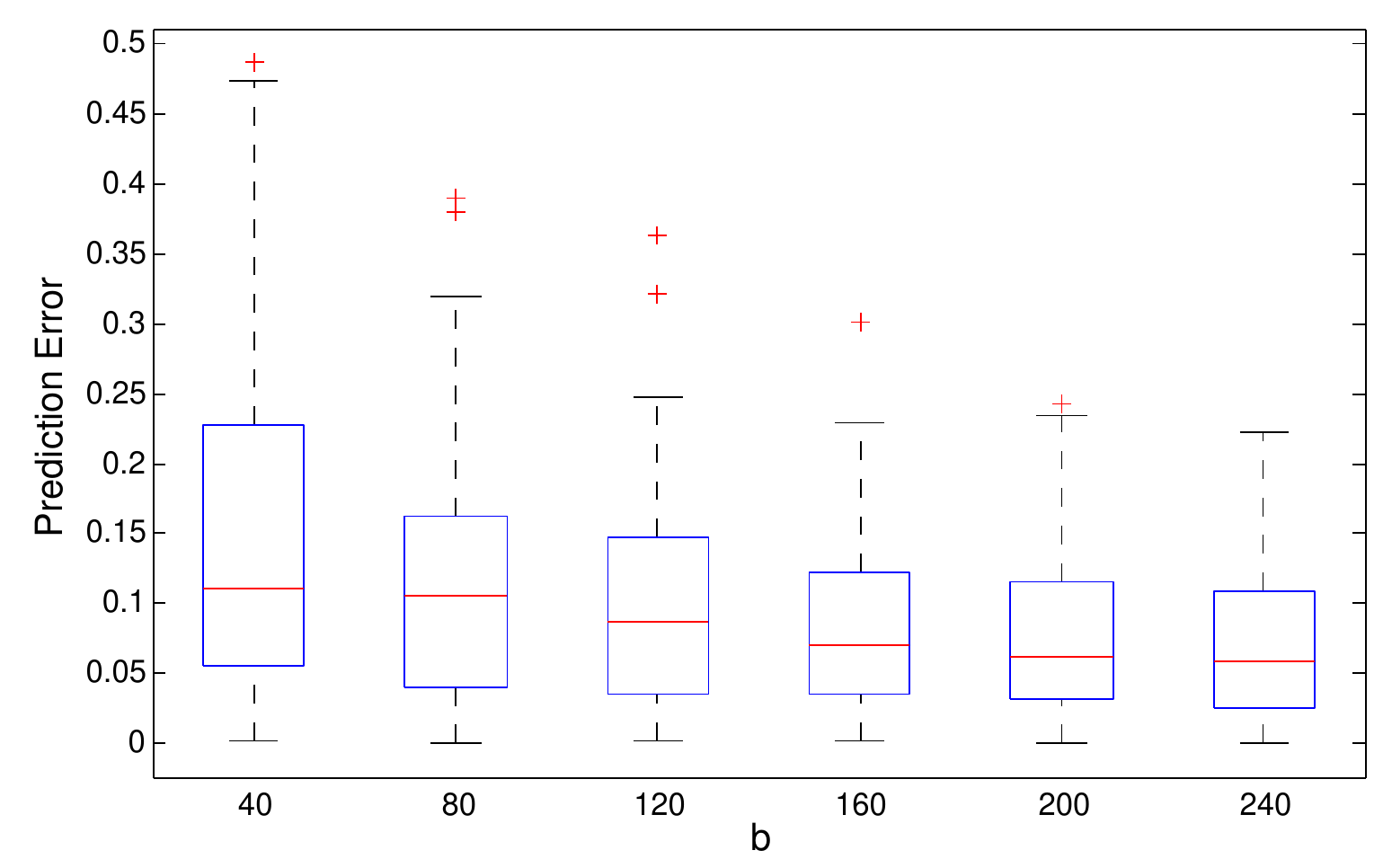}
\end{tabular}
\caption{Aircraft engine prognostic example: box-plot for relative prediction error of the estimated life time of the engine versus threshold $b$. }
\label{fig:system1}
\vspace{-0.1in}
\end{center}
\end{figure}

\section{Discussion: adaptive choice of $p_0$}\label{sec:con}

The mixture procedure assumes that a fraction $p_0$ of the sensors are affected by the change. In practice, $p_0$ can be different from $p$ which is the actual fraction of sensors affected. The performance of the procedure is fairly robust to the choice of $p_0$. Fig. \ref{robustDD} compares the simulated EDD of a mixture procedure with a fixed $p_0$ value, versus a mixture procedure when setting $p_0 = p$ if we know the true fraction of affected sensors. Again, thresholds are chosen such that ARL for all cases are 5000. Note that the detection delay is the smallest if $p_0$ matches $p$; however,  EDD in these two settings are fairly close when $p_0\neq p$.

\begin{figure}[h]
\begin{center}
\includegraphics[width = 0.45\linewidth]{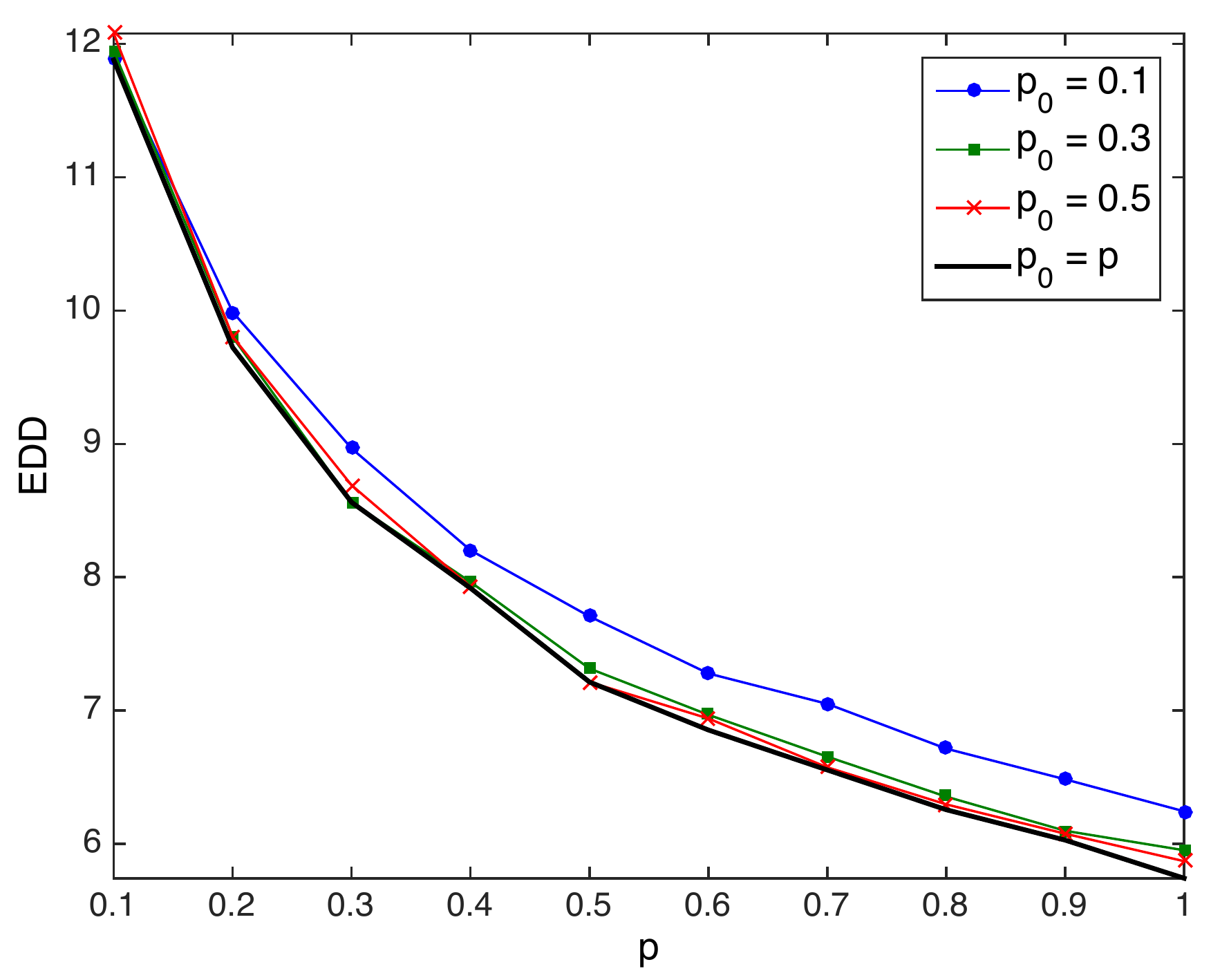}
\caption{Simulated EDD for a mixture procedure with $p_0$ set to a fixed value, versus a mixture procedure with $p_0 = p$ equal to the true fraction of affected sensors, when $c_n = 0.1$, $N=100$ and $w=200$. }
\label{robustDD}
\end{center}
\end{figure}

Still, we may improve the performance of the mixture procedure by adapting the parameter $p_0$ using a method based on empirical Bayes. Assume each sensor is affected with probability $p_0$, but now $p_0$ itself is a random variable with Beta distribution $\mbox{Beta}(\alpha, \beta)$. This also allows the probability of being affected to be different at each sensor. With sequential data, we may update by computing a posterior distribution of $p_0$ using data in the following way.
Choosing a constant $a$, we believe that the $n$th sensor is likely to be affected by the change-point if $U_{n,k,t}$ is larger than $a$. Let $\mathbb{I}\{\cdot\}$ denote an indicator function. For each $t$, assume
$s_{n, t} = \mathbb{I}\left\{\max_{t-w\leq k <t} U_{n,k,t}>a\right\}$ is a Bernoulli random variable with parameter $p_n$. Due to conjugacy, the posterior of $p_0$ at the $n$th sensor, given $s_{n, t}$ up to time $t$, is also a Beta distribution with parameters $\mbox{Beta}(s_{n, t}+\alpha,1-s_{n, t}+\beta)$. An adaptive mixture procedure can be formed using the posterior mean of $p_0$, which is given by $\rho_n \triangleq (s_{n, t}+\alpha)/(\alpha +\beta + 1)$:
\begin{equation}
T_{\rm adaptive}=\inf \left\{ t : \max_{t-w\leq k <t} \sum_{n=1}^N \log(1-\rho_n+\rho_n\exp(U_{n, k, t}^2/2))  \geq b \right\},
\label{adaptivestatistic}
\end{equation}
where $b$ is a prescribed threshold.

We compare the performance of $\widetilde{T}_{\rm adaptive}$ with its non-adaptive counterpart $\widetilde{T}_2$ by numerical simulations. Assume $N = 100$ and there are 10 sensors affected from the initial time with a rate-of-change $c_n = c$. The parameters for $\widetilde{T}_{\rm adaptive}$ are $\alpha=1,\beta=1$ and $a=2$.
Again, the thresholds are set so that the simulated ARL for both procedures are 5000. Table \ref{comparsionofDD1}  shows that  $\widetilde{T}_{\rm adaptive}$ has a much smaller EDD than $\widetilde{T}_2$ when signal is weak with a relative improvement around $20\%$.
However, it is more difficult to analyze ARL of the adaptive method theoretically.

%

\begin{table}[h]
\begin{center}
\caption{Comparing EDD of $\widetilde{T}_2$ and $\widetilde{T}_{\rm adaptive}$.}
  \begin{tabular}{ |c | c | c | c | c | c|}
    \hline
    Rate-of-change & 0.01 & 0.03 & 0.05 & 0.07 & 0.09 \\ \hline
    Non-Adaptive $\widetilde{T}_2$ & 54.15 & 26.24 & 18.75 & 14.98 & 12.74 \\ \hline
    Adaptive $\widetilde{T}_{\rm adaptive}$ & 38.56 & 20.28 & 14.42 & 12.17 & 10.13 \\
    \hline
  \end{tabular}
\end{center}
\label{comparsionofDD1}
\end{table}

%

\section*{Acknowledgement}
Authors would like to thank Professor Shi-Jie Deng at Georgia Tech for providing the financial time series data. This work is partially supported by NSF grant CCF-1442635 and CMMI-1538746.

\bibliographystyle{spmpsci}
\bibliography{slope_letter,AuthorResponse}

\appendix

\section{An informal derivation of Theorem \ref{maintheorem}: ARL}

We first obtain an approximation to the probability that the stopping time is greater than some big constant $m$. Such an approximation is obtained using a general method for computing first passing probabilities first introduced in \cite{yakir2013extremes} and developed in \cite{siegmund2011detecting}. The method relies on measure transformations that shift the distribution of each sensor over a window that contains the hypothesized post-change samples. More technical details to make the proofs more rigorous are omitted. These details have been described and proved in \cite{siegmund2011detecting}.

In the following, let $\tau = t-k$. Define the log moment-generating-function $\psi_{\tau}(\theta) = \log \mathbb{E} \exp\{\theta g(U_{n, k, t})\}$. Recall that $U_{n, k, t}$ is a generic standardized sum over all observations within a window of size $\tau$ in one sensor, and the parameter $\theta = \theta_\tau$ is selected by solving the equation
\[\dot{\psi}_\tau(\theta) = b/N.\] Since $U_{n, k, t}$ is a standardized weighted sum of $\tau$ independent random variables, $\psi_\tau$ converges to a limit as $\tau \rightarrow \infty$, and $\theta_\tau$ converges to a limiting value. We denote this limiting value by $\theta$.

Denote the density function under the null as $\mathbb{P}$. The transformed distribution for all sequences at a fixed current time $t$ and at a hypothesized change-point time $k$ (and hence there are $\tau$ hypothesized post-change samples) is denoted by $\mathbb{P}_t^k$ and is defined via
\[
d\mathbb{P}_t^k = \exp\left[\theta_\tau \sum_{n=1}^N g(U_{n, k, t}) - N\psi_\tau(\theta_\tau)\right] d \mathbb{P}.
\]
Let
\[
\ell_{N, k, t} = \log (d \mathbb{P}_t^k/d\mathbb{P}) = \theta_\tau \sum_{n=1}^N g(U_{n, k, t}) - N\psi_\tau(\theta_\tau).
\]
Let the region
\[
D = \{(t, k): 0<t<t_0, 1\leq t-k \leq w\}
\]
be the set of all possible change-point times and time up to a horizon $m$. Let
\[
A = \left\{\max_{(t, k) \in D} \sum_{n=1}^N g(U_{n, k, t}) \geq b\right\}.
\]
be the event of interest. Hence, we have
\begin{equation}
\begin{split}
\mathbb{P}\{A\}  & = \sum_{(t, k) \in D} \mathbb{E}
\left\{
\frac{e^{\ell_{N, k, t}} }{\sum_{(t', k')\in D} e^{\ell_{N,k',t'}}}; A
\right\} = \sum_{(t, k) \in D} \mathbb{E}_t^k\left\{\left(\sum_{(t', k')\in D} e^{\ell_{N, k', t'}}\right)^{-1}; A\right\}\\
 & =\sum_{(t, k) \in D} e^{-N(\theta_\tau b - \psi_\tau(\theta_\tau))}\times \underbrace{\mathbb{E}_t^k\left\{
 \frac{M_{N, k, t}}{S_{N, k, t}} e^{-\tilde{\ell}_{N, k, t} - \log M_{N, k, t}};  \tilde{\ell}_N + \log M_{N, k, t} \geq 0
 \right\}}_{I}
\end{split}
\end{equation}
where
\[
\tilde{\ell}_{N, k, t} = \sum_{n=1}^N \theta_\tau [g(U_{n, k, t}) - b],
\]
\[
S_{N, k, t}= \sum_{(t', k') \in D} e^{\sum_{n=1}^N \theta_\tau [g(U_{n, k', t'}) - g(U_{n, k, t})]}, \]\[
M_{N, k, t} = \max_{(t', k') \in D} e^{\sum_{n=1}^N \theta_\tau [g(U_{n, k', t'}) - g(U_{n, k, t})]}.
\]
As explained in \cite{siegmund2011detecting}, under certain verifiable assumptions, a ``localization lemma'' allows simplifying the quantities of the form in $I$ into a much simpler expression of the form
\[
\sigma_{N, \tau}^{-1}(2\pi)^{-1/2} \mathbb{E}\{M/S\},
\]
where $\sigma_{N,\tau}$ is the $\mathbb{P}_s^\tau$ standard deviation of $\tilde{\ell}_N$ and $\mathbb{E}[M/S]$ is the limit of $\mathbb{E}\{M_{N,k,t}/S_{N,k,t}\}$ as $N\rightarrow \infty$. This reduction relies on the fact that, for large $N$ and $m$, the ``local'' processes $M_{N,k,t}$ and $S_{N,k,t}$ are approximately independent of the ``global'' process $\tilde{\ell}_N$. This allows the expectation to be decomposed into the expectation of $M_N/S_N$ times the expectation involving $\tilde{\ell}_N +\log M_N$, treating $\log M_N$ as a constant.

Let $\tau' = t'-k'$, and denote by $z_{n, i} = (y_{n, i}-\mu_n)/\sigma_n$ which are i.i.d. normal random variables, $i=1,2,\ldots$. Note that, use Taylor expansion up to the first order, we obtain
\begin{equation}
\begin{split}
&\sum_{n=1}^N \theta_\tau [g(U_{n, k', t'}) - g(U_{n, k, t})]
 \approx  \sum_{n=1}^N \theta_\tau \dot{g}(U_{n, k, t})[U_{n, k', t'} - U_{n, k, t}] \\
  = & \sum_{n=1}^N \theta_\tau \dot{g}(U_{n, k, t})
  [A_{\tau'}^{-1/2}  W_{n,k',t'}  - A_\tau^{-1/2} W_{n, k', t'}
+ A_\tau^{-1/2} W_{n, k', t'} - A_\tau^{-1/2} W_{n, k, t} ]\\
=&  \sum_{n=1}^N \frac{\theta_\tau \dot{g}(U_{n, k, t})}{\sqrt{A_{\tau'}}}(\sum_{j=1}^{\tau'} j z_{n, t'-\tau'+j} - \sqrt{\frac{A_{\tau'}}{A_\tau}} \sum_{j=1}^{\tau'} j z_{n, t'-\tau' + j}) \\
&\qquad  + \sum_{n=1}^N \theta_\tau \dot{g}(U_{n, k, t}) 
A_\tau^{-1/2}\left(\sum_{i=1}^{\tau'} i z_{n, t'-\tau' + i}- \sum_{i=1}^\tau i z_{n, t - \tau + i}\right)
\end{split}
\end{equation}
Note that in the above expression, the first term has two weighted data sequences running backwards from $t'$ and when $\tau$ and $\tau'$ both tends to infinity they tend to cancel with each other. Hence, asymptotically we need to consider the second term.
Observe that one may let $t'-k'=\tau$ and $\theta = \lim_{\tau \rightarrow \infty} \theta_\tau$ for $\theta_\tau$ in the definition of the increments and still maintain the required level of accuracy.  When $\tau = u$ the first term in the above expression, and the second term consists of two terms that are highly correlated. The second term can be rewritten as
\begin{equation}
\begin{split}
A_\tau^{-1/2} \theta_\tau \left[ \sum_{n=1}^N  \dot{g}(U_{n, k, t})  W_{n, k', t'} - \sum_{n'=1}^N  \dot{g}(U_{n', k, t})W_{n', k, t}\right].
\end{split}
\end{equation}
Since all sensors are assumed to be independent (or has been whitened by a known covariance matrix so the transformed coordinates are independent), so the covariance between the two terms is given by
\begin{equation}
\mbox{Cov}\left(\sum_{n=1}^N  \dot{g}(U_{n, k, t})W_{n, k, t},
\sum_{n'=1}^N  \dot{g}(U_{n', k, t})W_{n', k, t}
\right)
= \sum_{n=1}^N  [\dot{g}(U_{n', k, t})]^2 \mbox{Cov}(W_{n, k, t}, W_{n', k', t'}).
\end{equation}

%
For each $n$, let $k<k'<t<t'$ and $t-k = t'-k'=\tau$, and define $u \triangleq k'-k$ and $s \triangleq t-k'$. We have
\begin{equation}
\begin{split}
A_\tau^{-1}\mbox{Cov}(W_{n,k,t}, W_{n,k',t'})
=& \mathbb{E}\left\{ \frac{ \left(\sum_{i=k+1}^{t} (i-k)z_{n, i}\right) \left(\sum_{i=k'+1}^{t'} (i-k')z_{ni}\right) }{\sum_{i=1}^{\tau} i^2}\right\} \\
=& \mathbb{E}\left\{ \frac{ \sum_{i=k'+1}^{t} (i-k)(i-k')z_{n, i}^2 } {\sum_{i=1}^{\tau} i^2} \right\}
= \frac{\sum_{i=1}^{s} i^2+u\sum_{i=1}^{s}i}{\sum_{i=1}^{\tau} i^2}.
\end{split} \nonumber
\end{equation}
By choosing $u = \sqrt{\tau}$, we know that the expression above is approximately on the order of
\[
1-\frac{(k'-k)+(t'-t)}{2\left(\frac{2}{3}\tau^2+\frac{1}{3}\tau\right)} \approx 1- \frac{(k'-k)+(t'-t)}{\frac{4}{3}\tau^2}.
\]
Let $\eta \triangleq \frac{4}{3}\tau^2$.
Hence, by summarizing the derivations above and applying the law of large number, we have that when $N\rightarrow \infty$ and $\tau\rightarrow \infty$, the covariance between the two terms become
\[
\mbox{cov}\left(\sum_{n=1}^N \theta_\tau g(U_{n, k', t'}), \sum_{n'=1}^N \theta_\tau g(U_{n', k, t})\right) \approx \theta^2N\cdot [1- \frac{1}{\eta} (k'-k) - \frac{1}{\eta} (t'-t)].
\]
This shows that the two-dimensional random walk decouples in the change-point time $k'$ and the time index $t'$ and the variance of the increments in these two directions are the same and are both equal to $\theta^2 N/\eta$. Hence, the random walk along these two coordiates are asymptotically independent and it becomes similar to the case studied in  \cite{siegmund2011detecting}. Compare this with (the equation following equation (A.4) in \cite{siegmund2011detecting}), note that the only difference is that here the variance of the increment is proportional to $3/(4\tau^2)$ instead of $\tau$, so we may follow a similar chains of calculation as in
the proof in Chapter $7$ of \cite{yakir2013extremes},  \cite{siegmund2011detecting} \cite{xie2013sequential}, the final result corresponds to modifying the upper and lower limit by changing the window length expression to be $\sqrt{4/3}$ and $\sqrt{4w/3}$. 

%

\section{An informal derivation of Theorem \ref{thm:EDD}: EDD}

Recall that $U_{n,k,t}$ is defined in (\ref{U_def}), let $z_{n, i} = (y_{n, i} - \mu_n)/\sigma_n$. Then for $n \in \mathcal{A}$,  $z_{n, i}$ are i.i.d.  normal random variables with mean $c_n i/\sigma_n$ and unit variance, and for $n \in \mathcal{A}^c$, $z_{n, i}$ are i.i.d. standard normal random variables. Since we may write
\begin{equation}
U_{n,k,t} = \frac{\sum_{i=k+1}^t (i-k) z_{n, i}}{\sqrt{\sum_{i=k+1}^t (i-k)^2}}.
\end{equation}
For any time $t$ and $n\in \mathcal{A}$, we have
\begin{equation}
 \mathbb{E}_0^{\mathcal{A}}\{U_{n,0,t}^2\} = 1 + \left(\frac{c_n}{\sigma_n}\right)^2 \sum_{i=1}^{t} i^2
= 1+\left(\frac{c_n}{\sigma_n}\right)^2 \frac{t(t+1)(2t+1)}{6}
= \left(\frac{c_n}{\sigma_n}\right)^2 \frac{t^3}{3} + o(t^3),
\end{equation}
which grows cubically with respect to time. For the unaffected sensors, $n \in \mathcal{A}^c$, $\mathbb{E}_0^{\mathcal{A}}\{U_{n, 0, t}^2\} = 1$. Hence, the value of the detection statistic will be dominated by those affected sensors.

On the other hand, note that when $x$ is large,
\[g(x) = \log (1-p_0 + p_0e^{x^2/2}) = \log p_0  + \frac{x^2}{2} + \log\left(\frac{1-p_0}{p_0} e^{-x^2/2}\right) \approx \frac{x^2}{2} +  \log p_0. \]
Then the expectation of the statistic in (\ref{mainstatistic}) can be computed if $w$ is sufficiently large (at least larger than the expected detection delay), as follows:
\[
\mathbb{E}_0^{\mathcal{A}}\left\{\max_{k < t} \sum_{n=1}^N g \left( U_{n,k,t} \right)\right\} \approx \left(|\mathcal{A}| \log p_0 +\frac{1}{2} \sum_{n\in \mathcal{A}} \mathbb{E}_0^{\mathcal{A}}\left\{U_{n,k,t}^2\right\} + \frac{(N-|\mathcal{A}|)}{2}\right),
\]
%
At the stopping time, if we ignore of the overshoot of the threshold over $b$,  the value statistic is $b$.
Use Wald's identify \cite{Siegmund1985} and if we ignore the overshoot of the statistic over the threshold $b$, we may obtain a first order approximation as $b \rightarrow \infty$, by solving
\begin{equation}
|\mathcal{A}|\log p_0 + \frac{N-|\mathcal{A}|}{2}+\frac{ \mathbb{E}_0^{\mathcal{A}}\{T^3\}}{6} \left[\sum_{n \in \mathcal{A}}  \left(\frac{c_n}{\sigma_n}\right)^2 \right] = b.
\end{equation}
From Jensen's inequality, we know that
$\mathbb{E}_0^{\mathcal{A}}\{T_2^3\}  \geq (\mathbb{E}_0^{\mathcal{A}}\{T_2\})^3$.
Therefore, a first-order approximation for the expected detection delay is given by
\begin{equation}
\mathbb{E}_0^{\mathcal{A}}\{T_2\} \leq \left( \frac{b -N \log p_0 - (N - |\mathcal{A}|)\mathbb{E}\{g(U)\}}{\Delta^2/6} \right)^{1/3} + o(1).
\label{firstorderEDD}
\end{equation}

\section{Proof for Optimality}

\begin{proof}[Proof of Theorem \ref{lowerbound_multi}]
The proof starts by a change of measure from $\mathbb{P}_{\infty}$ to $\mathbb{P}_{k}^{\mathcal{A}}$. For any stopping time $T\in C(\gamma)$, we have that for any $K_{\gamma}>0$, $C>0$ and $\varepsilon \in (0,1)$,
\begin{equation}
\begin{split}
&~\mathbb{P}_{\infty} \left\{ k<T<k+(1-\varepsilon)K_{\gamma} \middle| T>k\right\} \\
=&~\mathbb{E}_{k}^{\mathcal{A}}\left\{ \mathbb{I}_{\{k<T<k+(1-\varepsilon)K_{\gamma} \}} \exp(-\lambda_{{\mathcal{A}},k,T}) \middle| T>k\right\} \\
\geq&~ \mathbb{E}_{k}^{\mathcal{A}} \left\{ \mathbb{I}_{\{k<T<k+(1-\varepsilon)K_{\gamma},\lambda_{{\mathcal{A}},k,T}<C \}}\exp(-\lambda_{{\mathcal{A}},k,T})\middle|T>k\right\} \\
\geq&~ e^{-C}\mathbb{P}_{k}^{\mathcal{A}}\left\{ k<T<k+(1-\varepsilon)K_{\gamma}, \max_{k<j<k+(1-\varepsilon)K_{\gamma}} \lambda_{{\mathcal{A}},k,j}<C\middle|T>k \right\}\\
\geq&~ e^{-C}\left[ \mathbb{P}_{k}^{\mathcal{A}} \left\{ T<k+(1-\varepsilon)K_{\gamma} \middle|T>k \right\} - \right. \\
&~\left. \mathbb{P}_{k}^{\mathcal{A}}\left\{\max_{1\leq j <(1-\varepsilon)K_{\gamma}} \lambda_{{\mathcal{A}},k,k+j}\geq C\middle|T>k \right\}  \right],
\label{chainrule}
\end{split}
\end{equation}
where $\mathbb{I}_{\{A\}}$ is the indicator function of any event $A$, the first equality is Wald's likelihood ratio identity and the last inequality uses the fact that for any event $A$ and $B$ and probability measure $\mathbb{P}$, $\mathbb{P}(A \bigcap B) \geq \mathbb{P}(A)-\mathbb{P}(B^c)$.

From (\ref{chainrule}) we have for any $\varepsilon \in (0,1)$
\begin{equation}
\mathbb{P}_{k}^{\mathcal{A}} \left\{ T<k+(1-\varepsilon)K_{\gamma} |T>k \right\} \leq p_{\gamma,\varepsilon}^{(k)}(T) + \beta_{\gamma,\varepsilon}^{(k)}(T),
\end{equation}
where
\begin{align*}
p_{\gamma,\varepsilon}^{(k)}(T) &= e^C \mathbb{P}_{\infty} \left\{ T<k+(1-\varepsilon)K_{\gamma} \middle|T>k \right\},\\
\beta_{\gamma,\varepsilon}^{(k)}(T) &= \mathbb{P}_{k}^{\mathcal{A}}\left\{\max_{1\leq j <(1-\varepsilon)K_{\gamma}} \lambda_{{\mathcal{A}},k,k+j}\geq C\middle|T>k \right\}.
\end{align*}
Next, we want to show that both $p_{\gamma,\varepsilon}^{(k)}(T)$ and $\beta_{\gamma,\varepsilon}^{(k)}(T)$ converge to zero for any $T\in C(\gamma)$ and any $k\geq 0$ as $\gamma$ goes to infinity.

First, choosing $C=(1+\varepsilon)I[(1-\varepsilon) K_{\gamma}]^q$, then we have
\begin{equation}
\begin{split}
\beta_{\gamma,\varepsilon}^{(k)}(T)
=& ~\mathbb{P}_{k} \left\{ \left[ (1-\varepsilon)K_{\gamma} \right]^{-q} \max_{1\leq j <(1-\varepsilon)K_{\gamma}} \lambda_{{\mathcal{A}},k,k+j}\geq (1+\varepsilon)I \middle|T>k \right\} \\
\leq & ~\mbox{esssup}~ \mathbb{P}_{k}^{\mathcal{A}} \left\{ \left[ (1-\varepsilon)K_{\gamma} \right]^{-q} \max_{1\leq j <(1-\varepsilon)K_{\gamma}} \lambda_{{\mathcal{A}},k,k+j}\geq (1+\varepsilon)I \middle|\mathcal{F}_{k} \right\}.
\end{split}
\end{equation}
By the assumption (\ref{secondassumption_multi}), we have
\begin{equation}
\sup_{0\leq k < \infty} \beta_{\gamma,\varepsilon}^{(k)} \xrightarrow[\gamma \rightarrow \infty]{} 0.
\end{equation}

Second, by Lemma 6.3.1 in \cite{tartakovsky2014sequential}, we know that for any $T\in C(\gamma)$ there exists a $k \geq 0$, possibly depending on $\gamma$, such that
$$
\mathbb{P}_{\infty} \left\{ T<k+(1-\varepsilon)K_{\gamma} \middle| T>k \right\} \leq (1-\varepsilon)K_{\gamma} / \gamma.
$$
Choosing $K_{\gamma} = (I^{-1}\log \gamma)^{1/q}$, then we have
$$
C = (1+\varepsilon)I(1-\varepsilon)^q I^{-1}\log \gamma = (1-\varepsilon^2)(1-\varepsilon)^{q-1} \log\gamma,
$$
and therefore,
\begin{equation}
\begin{split}
p_{\gamma,\varepsilon}^{(k)}(T) \leq&~ \gamma^{(1-\varepsilon^2)(1-\varepsilon)^{q-1} } (1-\varepsilon)K_{\gamma}/\gamma \\
=&~ (1-\varepsilon)(I^{-1}\log \gamma)^{1/q} \gamma^{(1-\varepsilon^2)(1-\varepsilon)^{q-1} -1} \xrightarrow[\gamma \rightarrow \infty]{} 0,
\end{split}
\end{equation}
where the last convergence holds since for any $q \geq 1$ and $\varepsilon \in (0,1)$ we have $(1-\varepsilon^2)(1-\varepsilon)^{q-1} < 1$.
Therefore, for every $\varepsilon \in (0,1)$ and for any $T\in C(\gamma)$ we have that for some $k \geq 0$,
$$
\mathbb{P}_{k}^{\mathcal{A}} \left\{ T<k+(1-\varepsilon)K_{\gamma} \middle| T>k \right\} \xrightarrow[\gamma \rightarrow \infty]{} 0,
$$
which proves (\ref{firstresult_multi}).

Next, to prove (\ref{ESM_multi}), since
$$
\mbox{ESM}_m^{\mathcal{A}}(T) \geq \mbox{SM}_m^{\mathcal{A}}(T) \geq \sup_{0\leq k < \infty} \mathbb{E}_{k}^{\mathcal{A}} \left\{ [(T-k)^+]^m \middle| T>k\right\},
$$
it is suffice to show that for any $T\in C(\gamma)$,
\begin{equation}
 \sup_{0\leq k < \infty} \mathbb{E}_{k}^{\mathcal{A}} \left\{ [(T-k)^+]^m \middle| T>k\right\} \geq [I^{-1}\log \gamma]^{m/q}(1+o(1)) ~\mbox{as}~ \gamma \xrightarrow[]{} 0,
\label{tempproofpart_single}
\end{equation}
where the residual term $o(1)$ does not depend on $T$.
Using the result (\ref{firstresult_multi}) just proved, we can have that for any $\varepsilon \in (0,1)$, there exists some $k \geq 0$ such that
$$
\inf_{T\in C(\gamma)} \mathbb{P}_{k}^{\mathcal{A}}\left\{ T-k\geq (1-\varepsilon)(I^{-1}\log \gamma)^{\frac{1}{q}}\middle|T>k\right\} \xrightarrow[\gamma \rightarrow \infty]{} 1.
$$
Therefore, by also Chebyshev inequality, for any $\varepsilon \in (0,1)$ and $T\in C(\gamma)$, there exist some $k \geq 0$ such that
\begin{equation}
\begin{split}
&\mathbb{E}_{k}^{\mathcal{A}} \left\{ [(T-k)^+]^m \middle| T>k\right\} \\
\geq & \left[(1-\varepsilon)(I^{-1}\log \gamma)^{\frac{1}{q}}\right]^m \mathbb{P}_{k}^{\mathcal{A}}\left\{ T-k\geq (1-\varepsilon)(I^{-1}\log \gamma)^{\frac{1}{q}}\middle| T>k\right\}\\
\geq & \left[(1-\varepsilon)^m (I^{-1}\log \gamma)^{m/q}\right] (1+o(1)), ~\mbox{as}~ \gamma \xrightarrow[]{} \infty,
\label{finalproofpart_single}
\end{split}
\end{equation}
where the residual term does not depend on $T$. Since we can arbitrarily choose $\varepsilon \in (0,1)$ such that the (\ref{finalproofpart_single}) holds, so we have (\ref{tempproofpart_single}), which completes the proof.

\end{proof}

\begin{proof}[Proof of Lemma \ref{ARL2FA_mixtureCUSUM}]
Rewrite $T_{\rm CS}(b)$ as
\begin{equation}
T_{\rm CS}(b) = \inf\left\{t: \max_{0\leq k<t} \prod_{n=1}^N\left(1-p_0+p_0 \exp(\lambda_{n,k,t})  \right) \geq e^b \right\}
\end{equation}
and define $T_{\rm SR}(b)$ an extended Shiryaev-Roberts (SR) procedure as follows:
\begin{equation}
T_{\rm SR}(b) = \inf\left\{ t: R_{t} \geq e^b\right\},
\end{equation}
where
$$
R_{t} = \sum_{k=1}^{t-1} \prod_{n=1}^N \left(1-p_0+p_0 \exp(\lambda_{n,k,t}) \right), t=1,2,\ldots; R_0=0.
$$
Clearly, $T_{\rm CS}(b)\geq T_{\rm SR}(b)$. Therefore, it is sufficient to show that $T_{\rm SR}(b) \in C(\gamma)$ if $b\geq \log \gamma$.

Noticing the martingale properties of the likelihood ratios, we have
\begin{equation}
\mathbb{E}_{\infty} \left\{\exp(\lambda_{n,k,t})\middle| \mathcal{F}_{t-1}\right\} = 1
\label{martingaleprop}
\end{equation}

for all $n=1,2,\ldots,N$, $t>0$ and $0\leq k <t$.
Moreover, noticing that
\begin{equation}
R_t = \sum_{k=1}^{t-2}\prod_{n=1}^N \left(1-p_0+p_0\exp(\lambda_{n,k,t-1} + \lambda_{n,t-1,t}) \right) + \prod_{n=1}^N \left(1-p_0+p_0\exp(\lambda_{n,t-1,t})\right),
\end{equation}
then combining (\ref{martingaleprop}) we have for all $t>0$,
\begin{equation}
\begin{split}
\mathbb{E}_{\infty} \left\{R_t\middle| \mathcal{F}_{t-1}\right\} =&  \sum_{k=1}^{t-2}\prod_{n=1}^N \left(1-p_0+p_0 \exp(\lambda_{n,k,t-1})  \cdot 1\right) + 1 \\
=& R_{t-1} + 1.
\end{split}
\end{equation}
Therefore, the statistic $\left\{ R_t-t\right\}_{t>0}$ is a $(P_{\infty},\mathcal{F}_{t})$-martingale with zero mean. If $\mathbb{E}_{\infty}\left\{T_{SR}(b)\right\} = \infty$ then the theorem is naturally correct, so we only suppose that $\mathbb{E}_{\infty}\left\{T_{\rm SR}(b)\right\} < \infty$ and thus $\mathbb{E}_{\infty} \left\{R_{T_{\rm SR}(b)}-T_{\rm SR}(b)\right\}$ exists. Next, since $0\leq R_t <e^b$ on the event $\{ T_{\rm SR}(b)>t\}$, we have
$$
\liminf_{t\rightarrow \infty} \int_{\{T_{\rm SR}(b)>t\}}|R_t-t| ~d\mathbb{P}_{\infty} =0.
$$
Now we can apply the optional sampling theorem to have $\mathbb{E}_{\infty}\left\{R_{T_{\rm SR(b)}}\right\} = \mathbb{E}_{\infty}\left\{T_{\rm SR}(b)\right\}$. By the definition of stopping time $T_{\rm SR}(b)$, we have $R_{T_{\rm SR}(b)} > e^b$. Thus, we have $\mathbb{E}_{\infty}\left\{T_{\rm CS}(b)\right\} \geq \mathbb{E}_{\infty}\left\{T_{\rm SR}(b)\right\} > e^b$, which shows that $\mathbb{E}_{\infty}\left\{T_{\rm CS}(b)\right\} > \gamma$ if $b\geq \log \gamma$.

\end{proof}

\begin{proof}[Proof of Theorem \ref{optimality_mixture_CUSUM}]
First, we notice that if $b\geq \log \gamma$
$$
E_{\infty} \left\{\widetilde{T}_{\rm CS}(b)\right\} \geq E_{\infty} \left\{T_{\rm CS}(b)\right\} \geq \gamma.
$$
Therefore, by Theorem \ref{lowerbound_multi}, it is sufficient to show that if $b\geq \log \gamma$ and $b=\mathcal{O}(\log \gamma)$, then
\begin{equation}
\mbox{ESM}_m^{\mathcal{A}}(T_{\rm CS}(b)) \leq \left(\frac{\log \gamma}{I_{\mathcal{A}}}\right)^{m/q}(1+o(1)) ~\mbox{as}~ \gamma \rightarrow \infty.
\end{equation}
Equivalently, it is sufficient to prove that
\begin{equation}
\mbox{ESM}_m^{\mathcal{A}}(T_{\rm CS}(b)) \leq \left(\frac{b}{I_{\mathcal{A}}}\right)^{m/q}(1+o(1)) ~\mbox{as}~ b \rightarrow \infty.
\label{temp0_optimality_mixture_CUSUM}
\end{equation}
To start with, we consider a special case when $p_0=1$ in $T_{\rm CS}$ and denote it by
$$
T_{\rm CS2}(b) = \inf\left\{ t>0: \max_{0\leq k<t} \sum_{n=1}^N \lambda_{n,k,t} \geq b \right\}.
$$
Next, we will prove an asymptotical upper bound for the detection delay of $T_{CS2}(b)$.

Let
\begin{equation}
G_b = \left\lfloor \left(\frac{b}{I_{\mathcal{A}}(1-\varepsilon)}\right)^{1/q} \right\rfloor,
\label{G_b}
\end{equation}
and then $(G_b)^q \leq b/[I_{\mathcal{A}}(1-\varepsilon)]$.
Noticing that under $\mathbb{P}_{k}^{\mathcal{A}}$, we have $\sum_{n=1}^N \lambda_{n,k,t} = \lambda_{\mathcal{A},k,t}$ almost surely since the the log-likelihood ratios are $0$ for the sensors that are not affected. Therefore, by (\ref{assumption_optimality_mixture_CUSUM})  we can have that for any $\varepsilon \in (0,1)$, $t\geq 0$ and some sufficiently large $b$,
\begin{equation}
\begin{split}
&\sup_{0\leq k < t} \mbox{esssup}~ \mathbb{P}_{k}^{\mathcal{A}}\left\{ \sum_{n=1}^N \lambda_{n,k,k+G_b}< (G_b)^q I_{\mathcal{A}}(1-\varepsilon)\middle|\mathcal{F}_{k}\right\} \\
\leq&\sup_{0\leq k < t} \mbox{esssup}~ \mathbb{P}_{k}^{\mathcal{A}}\left\{ \lambda_{\mathcal{A}, k,k+G_b}< (G_b)^qI_{\mathcal{A}}(1-\varepsilon)\middle|\mathcal{F}_{k}\right\} \\
\leq &  \sup_{0\leq k < t} \mbox{esssup}~ \mathbb{P}_{k}^{\mathcal{A}}\left\{ \lambda_{\mathcal{A},k,k+G_b} < b\middle| \mathcal{F}_{k}\right\} \leq \varepsilon.
\label{temp1_optimality_mixture_CUSUM}
\end{split}
\end{equation}
Then, for any $k \geq 0$ and integer $l\geq 1$, we can use (\ref{temp1_optimality_mixture_CUSUM}) $l$ times by conditioning on $\left(X_{n,1},\ldots,X_{n,k+(l_0-1)G_b}\right), n=1,2,\ldots,N$ for $l_0=l,l-1,\ldots,1$ in succession (see \cite{lai1998information}) to have
\begin{equation}
\begin{split}
&~\mbox{esssup}~\mathbb{P}_{k}^{\mathcal{A}} \left\{T_{\rm CS2}(b)-k >lG_b \middle| \mathcal{F}_{k}\right\} \\
\leq &~\mbox{esssup}~\mathbb{P}_{k}^{\mathcal{A}} \left\{ \sum_{n=1}^N \lambda_{n,k+(l_0-1)G_b+1,k+l_0G_b}
,l_0=1,
\ldots,l \middle| \mathcal{F}_{k} \right\} \leq \varepsilon^l.
\label{temp2_optimality_mixture_CUSUM}
\end{split}
\end{equation}
Therefore, for sufficiently large $b$ and any $\varepsilon \in (0,1)$, we have
\begin{equation}
\begin{split}
\mbox{ESM}_m(T_{\rm CS2}(b))
\leq& ~\sum_{l=0}^{\infty} \left\{ [(l+1)G_b]^m - (lG_b)^m \right\} \cdot \\
&\qquad \sup_{0\leq k<\infty} \mbox{esssup}~ \mathbb{P}_{k} \left\{[(T_{\rm CS2}-k)^+]^m > (lG_b)^m\middle|\mathcal{F}_{k} \right\}\\
\leq&~ (G_b)^m \sum_{l=0}^{\infty} [(l+1)^m - l^m] \varepsilon^l \\
=& ~(G_b)^m(1+o(1)) ~ \mbox{as} ~ b \rightarrow \infty,
\label{temp3_optimality_single_CUSUM}
\end{split}
\end{equation}
where the first inequality can be known directly from the geometric interpretation of expectation of discrete nonnegative random variables and the last equality holds since for any given $m\geq 1$, $[(l+1)^m - l^m]^{1/l} \rightarrow 1$ as $l \rightarrow \infty$ so that the radius of convergence is $1$.
Using the fact that $(G_b)^m \leq [b/I(1-\varepsilon)]^{m/q}$ we prove (\ref{temp0_optimality_mixture_CUSUM}) for the case $p_0=1$.


Next, we will deal with the case when $p_0 \in (0,1)$.
Rewrite $T_{CS2}(b)$ as
$$
T_{CS2}(b) = \inf\left\{ t: \max_{0\leq k<t} \left( N\log p_0+ \sum_{n=1}^N \lambda_{n,t}^k\right) > b+N\log p_0 \right\},
$$
then
$$
T_{CS2}(b-N\log p_0) = \inf\left\{ t: \max_{0\leq k<t} \left( N\log p_0+ \sum_{n=1}^N \lambda_{n,t}^k\right) > b \right\}.
$$
Clearly, $\mbox{ESM}_m^{\mathcal{A}}(T_{\rm CS}(b)) \leq \mbox{ESM}_m^{\mathcal{A}}(T_{CS2}(b-N\log p_0))$, and thus
\begin{equation}
\mbox{ESM}_m^{\mathcal{A}}(T_{\rm CS}(b)) \leq \left(\frac{b-N\log p_0}{I_{\mathcal{A}}}\right)^{m/q}(1+o(1)).
\label{temp3_optimality_mixture_CUSUM}
\end{equation}
Therefore, we can claim that (\ref{temp0_optimality_mixture_CUSUM}) holds for any fixed $p_0 \in (0,1]$ since $N$ and $p_0$ are constants. If $b \geq \log \gamma$ and $b=\mathcal{O}(\log \gamma)$, $T_{\rm CS}(b)$ belongs to $C(\gamma)$ and $\mbox{ESM}_m^{\mathcal{A}}(T_{\rm CS})$ achieves its lower bound.


\end{proof}

\begin{proof}[Proof of Corollary \ref{optimality_window_mixture_CUSUM}]

The main steps are almost the same with that in the proof of Theorem \ref{optimality_mixture_CUSUM}. The only different thing is that we need the condition $w_{\gamma} \geq G_b$ (defined in (\ref{G_b})) in order to make (\ref{temp2_optimality_mixture_CUSUM}) be correct for any $k \geq 0$ and any integer $l\geq 1$. And the additional assumption (\ref{window_mixture_assumption}) ensures this.
\end{proof}

\begin{proof}[Proof of Lemma \ref{checkassuption}]
Consider testing problem (\ref{P1}), then for any $k \geq 0$ and $j\geq 1$,
$$
\lambda_{\mathcal{A},k,k+j} = \sum_{n\in\mathcal{A}}\frac{1}{\sigma_n^2}\sum_{i=k+1}^{k+j} \left\{c_n(i-k)(y_{n,i}-\mu_j) - \frac{c_n^2(i-k)^2}{2}\right\}.
$$
We define, for each $n\in \mathcal{A}$ and for all $l=1,\ldots,j$,
$$
X_{n,l}^{(k)} = \frac{1}{\sigma_n^2} \left\{c_n l(y_{n,l+k}-\mu_j) - \frac{c_n^2 l^2}{2}\right\}.
$$
Then we have
$$
\lambda_{\mathcal{A},k,k+j} =  \sum_{l=1}^j \sum_{n\in\mathcal{A}} X_{n,l}^{(k)} = \sum_{l=1}^j  X_{\mathcal{A},l}^{(k)},
$$
where we define $X_{\mathcal{A},l}^{(k)} = \sum_{n\in \mathcal{A}}X_{n,l}^{(k)}$.

Under probability measure $\mathbb{P}_{k}^{\mathcal{A}}$, we easily know that $(X_{\mathcal{A},l}^{(k)})_{l=1}^j$ are independent variables which follow normal distribution $N((l^2/2)\sum_{n\in\mathcal{A}}c_n^2 ,l^2\sum_{n\in\mathcal{A}}c_n^2 )$. Other simple computation tells us that
$$
\mathbb{E}_{k}^{\mathcal{A}}\left\{(X_{\mathcal{A},l}^{(k)})^2\right\} < \infty, ~\forall l=1,\ldots,j,
$$
and under probability measure $\mathbb{P}_{k}^{\mathcal{A}}$,
$$
\sum_{l=1}^{\infty} \mbox{Var}\left(\frac{X_{\mathcal{A},l}^{(k)}}{l^3}\right)<\infty,
$$
where $\mbox{Var}(X)$ denotes the variance of random variable $X$.
Therefore, combining Kronecker¡¯s lemma with the Kolmogorov convergence criteria, we have immediately a strong law of large numbers which tells us that
$$
\frac{1}{j^3} \lambda_{\mathcal{A},k,k+j} \xrightarrow[j\rightarrow \infty]{a.s.} \sum_{n \in \mathcal{A}} \frac{c_n^2}{6\sigma_n^2}.
$$
Finally, we complete the proof by using the fact that all the observations are independent.

\end{proof}

\begin{proof}[Proof of Lemma \ref{ARL2FA_multi_GCUSUM}]
First, define $y_{n,t}^k = \frac{\sum_{i=k+1}^t (y_{n,i}-\mu_j)}{\sigma_j\sum_{i=k+1}^{t} (i-k)^2 }$, then
\begin{equation}
\begin{split}
\mathbb{P}_{\infty} \left\{\widetilde{T}_2(b) > t_0\right\} \geq& ~\mathbb{P}_{\infty} \left\{\max_{0<t\leq t_0}\max_{\max(0,t-m_{\gamma})\leq k < t_0} \sum_{n=1}^N \frac{(y_{n,t_0}^k)^2}{2} < b \right\} \\
\geq&~ \left[\mathbb{P}_{\infty}  \{Y<2b\}\right]^{w_{\gamma}t_0},
\end{split}
\end{equation}
where $Y$ is a random variable with $\chi_N^2$ distribution.
Then, since $\widetilde{T}_2(b)$ is a non-negative discrete random variable, we have
\begin{equation}
\begin{split}
\mathbb{E}_{\infty}\left\{\widetilde{T}_2(b)\right\} =& \sum_{t_0=0}^{\infty} \mathbb{P}_{\infty} \left\{\widetilde{T}_2(b) > t_0\right\} \\
\geq & \sum_{t_0=0}^{\infty}\left[\mathbb{P}_{\infty}  \{Y<2b\}\right]^{w_{\gamma}t_0}
= \frac{1}{1-[\mathbb{P}_{\infty}  \{Y<2b\}]^{w_{\gamma}}}.
\end{split}
\end{equation}
Then if we can choose some $b$ so that
$$
\mathbb{P}_{\infty} \{Y\geq2b\} \leq 1-\left(1-\frac{1}{\gamma}\right)^{1/m_{\gamma}},
$$
we can claim that $\mathbb{E}_{\infty}\left\{\widetilde{T}_2(b)\right\} \geq \gamma$ and thus $\widetilde{T}_2(b) \in C(\gamma)$. To choose appropriate threshold $b$, we need use the tail bound for the $\chi_N^2$ distribution. Since $\chi_1^2$ is sub-exponential with parameter $(2\sqrt{N},4)$, it is well known that $\mathbb{P}_{\infty} \{Y\geq2b\} \leq \exp(-\frac{2b-N}{8})$ if $b\geq N$.
If we set
$$
b \geq \frac{N}{2}-4\log \left[ 1-\left(1-\frac{1}{\gamma}\right)^{1/m_{\gamma}} \right]
$$
then $\widetilde{T}_1(b) \in C(\gamma)$.
\end{proof}

\begin{proof}[Proof of Theorem \ref{optimality_here}]

By Lemma \ref{checkassuption}, we can use Theorem \ref{lowerbound_multi} to obtain a lower bound for the detection delays of arbitrary procedures in $C(\gamma)$. Specifically, for all $m\geq 1$,
\begin{equation}
\liminf_{\gamma \rightarrow \infty} \inf_{T\in C(\gamma)} \mbox{ESM}_m^{\mathcal{A}}(T) \geq \liminf_{\gamma \rightarrow \infty} \inf_{T\in C(\gamma)} \mbox{SM}_m^{\mathcal{A}}(T) \geq \left(\frac{\log \gamma}{I_{\mathcal{A}}}\right)^{m/q}.
\end{equation}

(i) Since $T_1(b)$ is a specified mixture CUSUM procedure for testing problem (\ref{P1}) and the observations are independent, the optimality is an immediate corollary from Theorem \ref{optimality_mixture_CUSUM}.

(ii) Since $\widetilde{T}_1(b)$ is a specified window-limited mixture CUSUM procedure for testing problem (\ref{P1}) and the observations are independent, the optimality is an immediate result from Corollary \ref{optimality_window_mixture_CUSUM}.

(iii) The assumption that $\log w_{\gamma} = o(\log \gamma)$ ensures that $b \geq  \frac{N}{2}-4\log \left[ 1-\left(1-\frac{1}{\gamma}\right)^{1/m_{\gamma}} \right]$ and $b=\mathcal{O}(\log \gamma)$ can be satisfied simultaneously.
Since the observations are independent, then $\mbox{ESM}_1^{\mathcal{A}}(\widetilde{T}_2(b)) = \mbox{SM}_1^{\mathcal{A}}(\widetilde{T}_2(b)) = \mathbb{E}_0^{\mathcal{A}}[\widetilde{T}_2(b)]$.  
The optimality of $\widetilde{T}_2(b)$ is an immediate result from Lemma \ref{ARL2FA_multi_GCUSUM} and the first order approximation of the detection delays in (\ref{firstorderEDD}).

\end{proof}

\end{document}